\documentclass[10pt]{article} 
\usepackage[preprint]{tmlr}

\usepackage{amsmath,amsthm,amsfonts,bm,bbm}

\def\eqref#1{equation~\ref{#1}}

\def\1{\bm{1}}

\DeclareMathAlphabet{\mathsfit}{\encodingdefault}{\sfdefault}{m}{sl}
\SetMathAlphabet{\mathsfit}{bold}{\encodingdefault}{\sfdefault}{bx}{n}

\newtheorem{assumption}{Assumption}
\newtheorem{theorem}{Theorem}

\usepackage[hidelinks,hypertexnames=false]{hyperref}
\usepackage{url}

\usepackage{graphicx}
\usepackage{float}
\usepackage{longtable}
\usepackage{xcolor}

\title{Conformalized Quantum DeepONet Ensembles: Towards Scalable Operator Learning with Distribution-Free Guarantees}

\author{\name Purav Matlia \email pmatlia@purdue.edu \\
    \addr Department of Computer Science\\
    Purdue University
    \AND
    \name Christian Moya \email cmoyacal@purdue.edu \\
    \addr Department of Mathematics \\
    Purdue University
    \AND
    \name Guang Lin \email guanglin@purdue.edu \\
    \addr Department of Mathematics \\
    Department of Mechanical Engineering \\
    Purdue University}

\def\month{MM}  
\def\year{YYYY} 
\def\openreview{\url{https://openreview.net/forum?id=XXXX}} 

\begin{document}

    \maketitle

    \begin{abstract}
        Operator learning enables fast surrogate modelling of high-dimensional dynamical systems, but existing approaches face two fundamental limitations: the quadratic cost of dense neural layers and unreliable uncertainty quantification in safety-critical settings. We propose Conformalized Quantum DeepONet Ensembles, a framework that addresses both challenges simultaneously. Using Quantum Orthogonal Neural Networks (QOrthoNNs), we characterize the resource regime in which their established $\widetilde{\mathcal{O}}(n)$ hidden-layer running-time scaling improves on the $\mathcal{O}(n^2)$ cost of a classical dense layer. To quantify uncertainty, we combine ensemble predictions with split conformal calibration. For a new input--output function pair jointly exchangeable with the calibration pairs, we prove a finite-sample, distribution-free lower bound on the expected fraction of covered query locations. As a secondary proof-of-concept, we explore hybrid classical--quantum architectures and superposed execution to manage ensemble runtime and hardware requirements. Experiments on synthetic operator benchmarks and real-world power-system dynamics show accurate predictions under ideal simulation and empirical coverage near the target under both ideal conditions and selected compact-circuit simulations with depolarizing or device-calibrated composite noise. Together, these results connect resource-aware scalability analysis with finite-sample uncertainty guarantees for quantum operator learning.
    \end{abstract}

    \section{Introduction}
    Evaluating dynamical systems across large sets of hypothetical scenarios is central to scientific computing and safety-critical applications such as power systems. While neural operator methods such as DeepONet provide fast surrogate models for these tasks, they face two fundamental challenges: (i) scalability, due to the quadratic $\mathcal{O}(n^2)$ cost of dense neural layers, and (ii) reliability, as standard uncertainty quantification methods either lack rigorous guarantees or incur prohibitive computational overhead. These limitations are particularly acute in settings that require both fine discretization and trustworthy predictions.

    Traditional numerical simulations accurately analyse dynamical systems governed by high-dimensional differential equations \citep{dahlquist2008numerical, butcher2016numerical, liu2019mcsimulation}. However, repeatedly evaluating these simulators for thousands of hypothetical perturbations is computationally expensive, motivating the use of learned surrogate operators \citep{lu2021deeponet, li2023neuraloperators, azizzadenesheli2024neuraloperators, wen2022neuraloperator} to approximate the solution map between infinite-dimensional function spaces directly.

    Deep Operator Networks (DeepONet) \citep{lu2021deeponet, lu2022deeponetcomparison, lu2022multifidelitydeeponet, lin2023responsedeeponet} approximate solution maps by decomposing them into branch and trunk subnetworks. However, classical hidden layers within these networks rely on dense matrix-vector multiplications that incur an $\mathcal{O}(n^2)$ computational cost \citep{kerenidis2022qorthonn, landman2022medicalqorthonn, goodfellow2016deeplearning}. For Monte Carlo uncertainty quantification, the DeepONet must be evaluated repeatedly across many realizations and query coordinates, causing this quadratic hidden-layer cost to accumulate \citep{zhang2021mccost, milanovic2017mccost, ye2023mccost}. This motivates sub-quadratic hidden-layer mechanisms.

    Recent work proposes scalable architectures like the Quantum DeepONet \citep{xiao2025qdon} and Quantum Fourier Neural Operator \citep{jain2024qfno}. By replacing $\mathcal{O}(n^2)$ classical matrix multiplications with parameterized Reconfigurable Beam Splitter (RBS) circuits operating in the unary subspace \citep{kerenidis2022qorthonn, landman2022medicalqorthonn, xiao2025qdon}, Quantum DeepONets reduce hidden-layer running time complexity to $\widetilde{\mathcal{O}}(n)$ conditional on a resource regime. Reduced hidden-layer scaling is particularly relevant in settings that use fine discretization and wide hidden representations to capture high-frequency transients. However, a single Quantum DeepONet yields only a noisy point prediction and does not provide the epistemic uncertainty estimates required for safety-critical deployment.

    Bayesian Neural Networks \citep{neal1996bnn, goan2020bnn} and Monte Carlo dropout \citep{gal2016dropout} can model this uncertainty, but neither provides the distribution-free coverage guarantees required for safety-critical tasks. Conformalized deep ensembles \citep{angelopoulous2023conformal, moya2025conformal} instead combine model disagreement as a proxy for epistemic uncertainty with distribution-free coverage guarantees. While methods such as Quantum Conformal Prediction (QCP) bypass ensembling costs \citep{park2023qcp} by treating measurement shots from a single circuit as samples from a predictive distribution, for Quantum DeepONets, shot variability reflects quantum noise rather than epistemic uncertainty. Nevertheless, using independently trained models introduces an ensemble-execution trade-off: sequential evaluation scales runtime with $L$, whereas fully parallel evaluation scales qubit requirements with $L$, motivating the execution strategies we explore briefly.
    
	Unlike prior work that considers quantum operator learning and conformal prediction separately, we jointly address Quantum DeepONet scalability and conformal uncertainty quantification through the following contributions:
	\begin{itemize}
		\item \textbf{Resource-aware hidden-layer complexity}. We refine the Quantum DeepONet complexity analysis of \citet{xiao2025qdon} by accounting for classical parallelism and quantum resources, and we leverage the established hidden-layer scaling of Quantum Orthogonal Neural Networks (QOrthoNNs). Under this model, QOrthoNNs provide an asymptotic running-time improvement over classical dense layers within an explicit resource regime. This is a complexity result, not a measured wall-clock comparison.
		\item \textbf{Uncertainty quantification}. We integrate split conformal calibration with ensemble DeepONet outputs. We prove a finite-sample, distribution-free lower bound on the expected fraction of covered query locations for a new input--output function pair under exchangeability of complete pairs. Query locations within a pair may be arbitrarily dependent.
		\item \textbf{Exploratory hybrid and multiplexed ensembling}. We explore hybrid classical--quantum architectures and a multiplexed circuit inspired by Superposed Parameterized Quantum Circuits (SPQCs) \citep{patapovich2025spqc} as proof-of-concept strategies for managing the runtime and quantum-resource costs of ensembling. The multiplexed circuit reduces qubit use relative to fully parallel execution but can increase circuit depth and shot requirements.
		\item \textbf{Evaluation under ideal and simulated-noise regimes}. We evaluate prediction accuracy and empirical query-averaged coverage across synthetic operator benchmarks and power-system tasks under ideal conditions, with selected compact-circuit studies under depolarizing and device-calibrated composite noise simulations.
	\end{itemize}

    \begin{figure}[H]

        \begin{center}
            \includegraphics[width=\columnwidth]{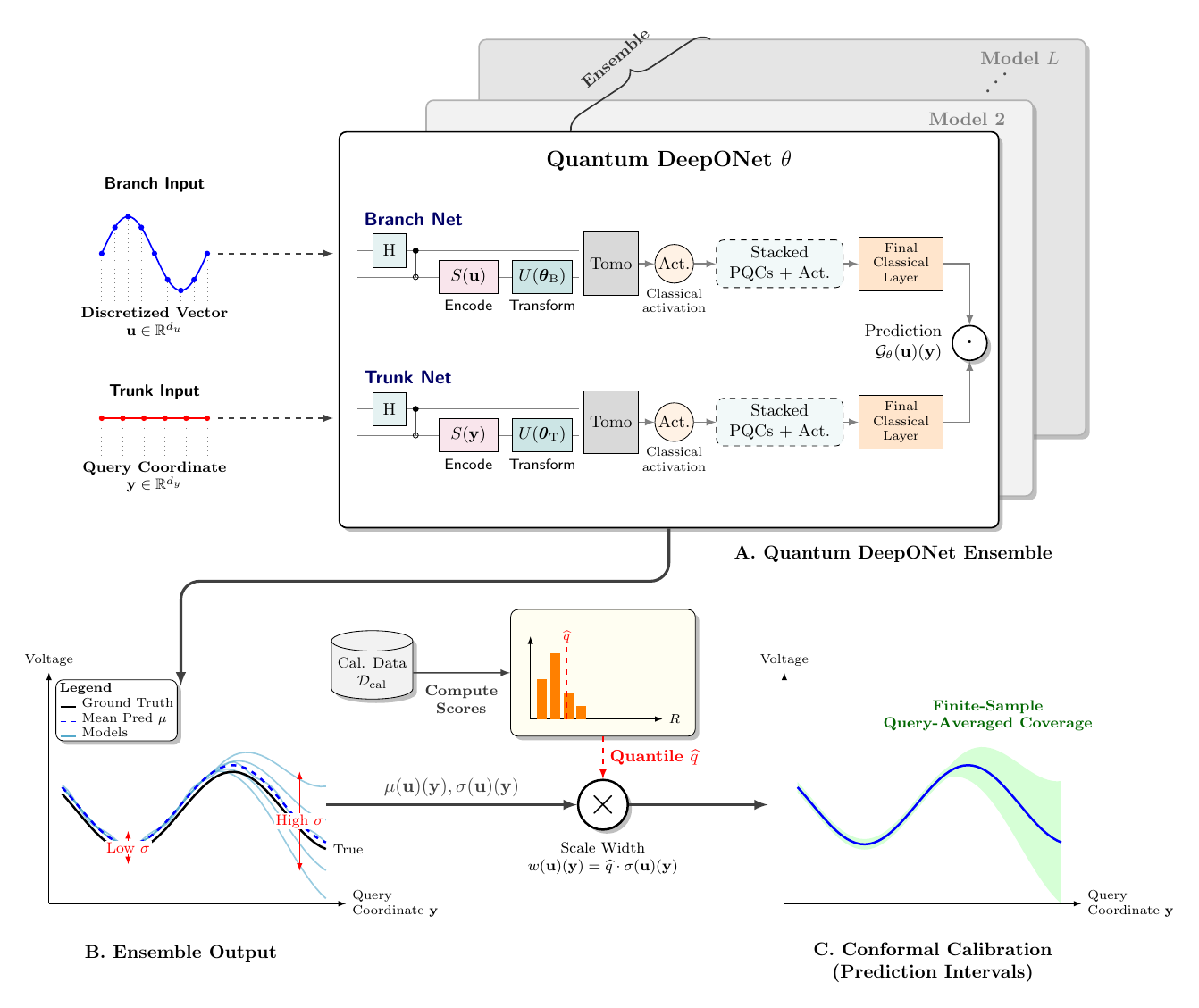}
        \end{center}

        \caption{Overview of the conformalized Quantum DeepONet ensemble framework. (A) Quantum DeepONet ensemble: independently trained operator models, each decomposed into branch and trunk subnetworks implemented using Quantum Orthogonal Neural Networks. Here, $\theta$ denotes the complete-model parameters, while $\boldsymbol{\theta}$ denotes the RBS-angle vector of a QOrthoNN layer; subscripts $\mathrm{B}$ and $\mathrm{T}$ identify branch and trunk layers, respectively. (B) Ensemble output: predictions are aggregated to compute the mean prediction $\mu(\mathbf{u})(\mathbf{y})$ and ensemble standard deviation $\sigma(\mathbf{u})(\mathbf{y})$, which serves as a model-disagreement scale. (C) Conformal calibration: using calibration data $\mathcal{D}_{\mathrm{cal}}$, nonconformity scores define a threshold $\widehat q$ and corresponding prediction intervals $C_{\alpha}(\mathbf{u},\mathbf{y})$, with a finite-sample lower bound on the expected fraction of covered query coordinates in a new input--output function pair that is jointly exchangeable with the calibration pairs.}
        \label{fig:main_framework}
    \end{figure}

	Figure~\ref{fig:main_framework} summarizes the resulting framework. The rest of the paper is organized as follows. Following a review of previous work in Section~\ref{sec:previous_work}, we formulate the problem in Section~\ref{sec:problemformulation}. Section~\ref{sec:method} details the resource-aware complexity analysis, quantum circuit design, ensemble construction, and exploratory execution strategies, and Section~\ref{sec:uq} introduces the conformal calibration procedure and establishes its finite-sample coverage bound. Section~\ref{sec:experiments} evaluates predictive accuracy and uncertainty calibration, and Section~\ref{sec:limitations} discusses limitations and future directions. A summary of notation is provided in Appendix~\ref{sec:nomenclature}.

    \section{Previous Work}
    \label{sec:previous_work}

    \paragraph{Operator learning and scientific machine learning.}
    The scientific machine learning community has introduced various data-driven frameworks for approximating the dynamic response of systems governed by ordinary and partial differential equations \citep{karniadakis2021physicsinformedml, baker2019scimlreport}. Early efforts focused on data-driven surrogates that either learn discrete-time evolution maps or infer continuous-time governing equations \citep{qin2019datadriven, raissi2018multistep}, and on sparse-regression schemes that recover governing equations directly from data \citep{brunton2016sindy}. To leverage physical priors and reduce data requirements, Physics-Informed Neural Networks (PINNs) \citep{raissi2019pinn} encode the governing equations as soft constraints in the loss, with extensions to stiff \citep{ji2020stiffpinn}, multi-scale \citep{leung2022nhpinn}, and differential--algebraic systems, including applications in power grids \citep{moya2023daepinn, huang2022pinnpower} and biology \citep{yazdani2020sysbiopinn}. Because PINNs typically require retraining whenever problem parameters change, attention has shifted to operator learning, in which neural architectures approximate mappings between infinite-dimensional function spaces. Deep Operator Networks (DeepONet) \citep{lu2021deeponet, lu2022deeponetcomparison, lu2022multifidelitydeeponet, lin2023responsedeeponet} and Fourier Neural Operators \citep{li2023neuraloperators, wen2022neuraloperator, azizzadenesheli2024neuraloperators} have demonstrated strong generalization across parametric PDE families and complex multi-physics applications, motivating their use as surrogates for safety-critical simulations such as power-system transient analysis \citep{moya2023deeponetgrid}.

	\paragraph{Quantum machine learning for scientific computing.}
	Parallel to these classical advances, quantum machine learning has been applied to differential equations through variational and hybrid quantum--classical formulations, including variational algorithms for nonlinear PDEs \citep{lubasch2020variational,joo2021quvapde}, differentiable quantum circuits and quantum kernel methods for differential equations \citep{kyriienko2021differentiable,paine2023quantumkernels}, quantum lattice-Boltzmann methods with multi-fidelity neural correction \citep{jacob2025multifidelityquantum}, and quantum analogues of physics-informed neural networks \citep{trahan2024qpinn,farea2025qcpinn}. Quantum Orthogonal Neural Networks (QOrthoNNs) \citep{kerenidis2022qorthonn, landman2022medicalqorthonn} encode data in the unary basis and realize orthogonal transformations with Reconfigurable Beam Splitter gates, and have been evaluated on medical-image classification tasks. Building on QOrthoNNs, Quantum DeepONet \citep{xiao2025qdon} extends operator learning to quantum settings, while the Quantum Fourier Neural Operator architecture \citep{jain2024qfno} uses a unary-basis quantum Fourier transform and has been evaluated across several benchmark PDE families. QuanONet \citep{wang2025quanonet} uses a hardware-efficient ansatz to learn solution operators for differential equations.

    \paragraph{Uncertainty quantification for scientific machine learning.}
    Reliable deployment of neural surrogates in safety-critical settings demands rigorous uncertainty quantification \citep{psaros2023uq, zou2024operatoruq}. Classical Bayesian neural networks \citep{mackay1992practical, neal1996bnn} model distributions over parameters but do not scale well or provide mathematical guarantees. Alternative approaches include deep ensembles \citep{lakshminarayanan2017deep}, Bayesian PINNs \citep{yang2021bpinn}, probabilistic CNNs for PDEs \citep{winovich2019convpdeuq}, and Bayesian formulations of DeepONet \citep{lin2023bdeeponet}. For a comprehensive review of these deep learning UQ techniques, we refer readers to \citet{abdar2021uqreview} and \citet{psaros2023uq}. These methods generally lack distribution-free coverage guarantees, motivating the adoption of conformal prediction \citep{vovk1999conformal, vovk2022conformallearning, angelopoulous2023conformal} to provide rigorous intervals for neural operators. Recent work further extends conformal methodology to the operator learning paradigm \citep{moya2025conformal}. In the quantum setting, \citet{park2023qcp} applied conformal prediction to quantum machine learning, but uncertainty-aware quantum operator learning remains largely unexplored. 

    \section{Problem Formulation}
    \label{sec:problemformulation}

   We aim to approximate a target operator $\mathcal{G}$ that maps input functions to pointwise outputs using a parameterized operator $\mathcal{G}_{\theta}$. We train $\mathcal{G}_{\theta}$ using a supervised dataset $\mathcal{D}=\{(\mathbf{u}_{i},\mathbf{y}_{ij},s_{ij})\}_{i=1,j=1}^{N,M_{i}}$, where $\mathbf{u}_{i}=(u_i(x_{1}),\dots,u_i(x_{d_u}))^{\top}\in\mathbb{R}^{d_u}$ represents a discretized input function evaluated at $d_u$ fixed sensor locations. The variable $\mathbf{y}_{ij}\in\mathbb{R}^{d_{y}}$ denotes a spatial or temporal query coordinate, and $s_{ij}=\mathcal{G}(\mathbf{u}_{i})(\mathbf{y}_{ij})\in\mathbb{R}$ represents the corresponding ground-truth value of the operator evaluated at that specific coordinate. To solve this operator learning problem, we use the Deep Operator Network (DeepONet) \citep{lu2021deeponet, lu2022deeponetcomparison, lu2022multifidelitydeeponet, lin2023responsedeeponet} as our foundational approximation architecture.

    The DeepONet architecture parameterizes $\mathcal{G}_{\theta}$ using two subnetworks: a branch network that encodes discretized input functions and a trunk network that embeds spatial or temporal query coordinates. For any generic input function $\mathbf{u}$ and query coordinate $\mathbf{y}$, the network approximates the target operator by computing the inner product of the $p$-dimensional branch output vector, $\mathbf{b}$, and trunk output vector, $\mathbf{t}$, and adding a trainable scalar bias $\beta$:
    \[\mathcal{G}_{\theta}(\mathbf{u})(\mathbf{y}) = \sum_{k=1}^{p} b_{k}t_{k}+\beta = \mathbf{b}^\top\mathbf{t}+\beta \approx \mathcal{G}(\mathbf{u})(\mathbf{y}) \in \mathbb{R}.\]
    We denote all trainable model parameters by $\theta \in \mathbb{R}^{d_\theta}$ and optimize them simultaneously by minimizing the Mean Squared Error (MSE) across the dataset, unless otherwise noted in Appendix~\ref{sec:train_hp}:
    \[\mathcal{L}(\theta) = \frac{1}{N} \sum_{i=1}^{N} \frac{1}{M_{i}} \sum_{j=1}^{M_{i}} (\mathcal{G}_{\theta}(\mathbf{u}_{i})(\mathbf{y}_{ij}) - s_{ij})^{2}\]
    where $d_\theta$ is the total number of parameters.
    
    Although this architecture is highly expressive, it introduces a computational bottleneck. Assuming an input size of $d_{\mathrm{in}}$ and a uniform hidden layer width of $n$, the initial layer maps the input dimension to $n$ at a cost of $\mathcal{O}(d_{\mathrm{in}} \cdot n)$. Subsequent dense hidden layers apply $\mathcal{O}(n^{2})$ matrix transformations followed by learned biases and nonlinear activations. At test time, when the learned operator must be evaluated across thousands of hypothetical conditions in a scenario batch, the $\mathcal{O}(n^2)$ cost of each dense layer accumulates across the batch.

	The DeepONet defined above produces point predictions but does not quantify uncertainty. Our uncertainty objective is to construct prediction sets $C_\alpha(\mathbf{u},\mathbf{y})$ and calibrate them toward the target level $1-\alpha$. Let $\mathbf{u}_*$ denote a new generic input function, $\mathbf{y}_{*,1},\ldots,\mathbf{y}_{*,M}$ the query coordinates, and $s_{*,1},\ldots,s_{*,M}$ the corresponding true values. We consider the expected fraction of covered values:
	\[\mathbb{E}\left[\frac{1}{M}\sum_{j=1}^{M}\mathbf{1}\left\{s_{*,j}\in C_\alpha(\mathbf{u}_*,\mathbf{y}_{*,j})\right\}\right].\]
	Equivalently, if $J\sim\operatorname{Unif}\{1,\ldots,M\}$ is an independently drawn query index, then
	\[\mathbb{E}\left[\frac{1}{M}\sum_{j=1}^{M}\mathbf{1}\left\{s_{*,j}\in C_\alpha(\mathbf{u}_*,\mathbf{y}_{*,j})\right\}\right]=\mathbb{P}\left\{s_{*,J}\in C_\alpha(\mathbf{u}_*,\mathbf{y}_{*,J})\right\}.\]
	We refer to this quantity as query-averaged marginal coverage. Because evaluations across query coordinates form a discretized output function, query-averaged coverage summarizes reliability across the predicted function without requiring simultaneous coverage at every coordinate. Simultaneous coverage is a stronger objective and generally requires wider prediction bands. These objectives motivate the resource-aware scalability analysis and conformalized Quantum DeepONet ensemble framework developed next.

    \section{Method}
    \label{sec:method}

	Our method uses QOrthoNN hidden layers in the branch and trunk networks of independently trained Quantum DeepONets, whose predictions are subsequently calibrated using conformal prediction. We begin by characterizing the hidden-layer scaling so that the scope of the computational claim is clear before the full architecture is introduced. We then describe the Quantum DeepONet architecture and its training and evaluation, followed by the ensemble quantities used for conformal calibration. Finally, because sequential ensemble execution increases runtime whereas fully parallel execution duplicates quantum resources, we explore hybrid classical--quantum and Superposed Parameterized Quantum Circuits (SPQC)-inspired multiplexed implementations as secondary strategies.

    \subsection{Resource-Aware Hidden-Layer Complexity}
    
	\paragraph{QOrthoNN construction and prior claim.} Quantum DeepONet replaces classical dense hidden transformations with Quantum Orthogonal Neural Network (QOrthoNN) layers \citep{kerenidis2022qorthonn,landman2022medicalqorthonn,xiao2025qdon}. A square width-$n$ QOrthoNN layer uses $n$ data qubits and one tomography ancilla and consists of a data loader $S(\cdot)$, a parameterized unitary $U(\cdot)$ arranged in a pyramidal circuit, and tomography. The data loader and parameterized circuit are built from Reconfigurable Beam Splitter (RBS) gates, which are also reused within tomography. RBS gates require only nearest-neighbour connectivity and admit an $\mathcal{O}(1)$ decomposition into standard basis gates. We analyse the square transformation here and defer the generalized $m\times n$ construction to Appendix~\ref{sec:quantum_state_evolution}. The circuit construction and depth estimates summarized below follow \citet{kerenidis2022qorthonn} and \citet{landman2022medicalqorthonn}; we refer readers to these works for the full derivations. Building on this analysis, \citet{xiao2025qdon} report a quantum-layer feedforward cost of $\mathcal{O}(n/\delta^2)$ against the $\mathcal{O}(n^2)$ cost of a classical dense layer and present this reduction as an advantage for frequent Quantum DeepONet evaluations. We refine their comparison by applying the complete-output recovery bound of \citet{landman2022medicalqorthonn} and accounting for classical parallelism, thereby identifying the resource regime in which the hidden-layer advantage holds.
	
	\paragraph{Quantum depth--shot cost.} The data loader has depth $n-1$, and the pyramidal transformation has depth at most $2n-3$. Tomography adds two data loaders and $\mathcal{O}(1)$ supplementary gates, leaving the complete circuit with $\mathcal{O}(n)$ depth. Recovering the $n$-component output so that every component has absolute error at most $\delta$ requires $\mathcal{O}(\log(n)/\delta^2)$ circuit repetitions \citep{landman2022medicalqorthonn}. Combined with the $\mathcal{O}(n)$ circuit depth, these repetitions yield a hidden-layer depth--shot cost of $\mathcal{O}(n\log(n)/\delta^2)$, or $\widetilde{\mathcal{O}}(n)$ for fixed $\delta$. The logarithmic factor comes from requiring the absolute recovery error to be at most $\delta$ simultaneously across all $n$ output components rather than for a single component; it accounts for the difference from the $\mathcal{O}(n/\delta^2)$ feedforward cost quoted above. This is a reduction in sequential depth, not total operations: each repetition performs the same quadratic number of planar rotations as the corresponding classical Givens-rotation implementation, but arranges disjoint RBS gates in parallel.
	
	\paragraph{Resource accounting.} Let $P=P(n)\in\mathbb{N}$ denote the classical processors assigned to a width-$n$ hidden-layer evaluation. Under the standard work--depth model, Brent's scheduling bound gives time $\mathcal{O}(W/P+D)$ for work $W$ and depth $D$ \citep{brent1974parallel,blelloch1996programming}. Ignoring memory movement and communication, a dense $n\times n$ matrix--vector product has $\mathcal{O}(n^2)$ work and $\mathcal{O}(\log(n))$ depth, giving parallel time $\mathcal{O}(n^2/P+\log(n))$. GPU kernels distribute rows or tiles across processing units, while distributed implementations can assign row blocks across additional nodes; the effective $P$ can therefore grow with $n$ until the allocated hardware is saturated. In contrast, \citet{xiao2025qdon} compare the quantum depth obtained by executing disjoint gates in parallel across $\mathcal{O}(n)$ qubits with $\mathcal{O}(n^2)$ sequential classical work, corresponding to $P=\mathcal{O}(1)$. Because the quantum model allows the qubit count to grow with $n$, we likewise allow $P$ to grow with $n$ and determine the processor-scaling regime in which the quantum advantage holds. We count the quantum repetitions sequentially on one width-$n$ circuit: simultaneous execution across multiple quantum processing units is not part of our model and would duplicate the $\mathcal{O}(n)$ circuit qubits. 
	
	\paragraph{Advantage regime.} Comparing the quantum depth--shot cost with the classical parallel time, the quantum bound grows more slowly when $P=o(n\delta^2/\log(n))$. For fixed $\delta$, this becomes $P=o(n/\log(n))$, which includes fixed or polylogarithmic processor allocations and sublinear polynomial growth. If $P$ instead grows proportionally with $n$, the classical bound becomes $\mathcal{O}(n)$ and the quantum depth--shot cost provides no asymptotic advantage.
	
	\paragraph{Scope.} This result concerns asymptotic hidden-layer scaling, not measured wall-clock performance. It omits hardware latency, memory and communication costs, and implementation-specific GPU, TPU, and quantum-hardware effects; we therefore do not claim a practical speedup on current hardware. Nor does the result establish an end-to-end asymptotic improvement. Here, $p$ denotes the branch and trunk output dimension defined in Section~\ref{sec:problemformulation}. Under the common width-preserving choice $p=n$, used throughout our experiments, the final map in each subnetwork is $n\times n$ and contributes $\mathcal{O}(n^2)$ arithmetic. The advantage is therefore restricted to QOrthoNN hidden transformations under the stated computational model and resource regime. We next describe the Quantum DeepONet architecture containing these transformations.

    \subsection{Quantum DeepONet Architecture}

    Each Quantum DeepONet consists of two subnetworks: a branch Quantum Orthogonal Neural Network (QOrthoNN) that processes the input function $\mathbf{u}$, and a trunk QOrthoNN that processes the query coordinate $\mathbf{y}$. Within these subnetworks, Parameterized Quantum Circuits (PQCs) execute orthogonal linear transformations, which are subsequently passed through classical nonlinear activations. To relax the orthogonality constraint imposed by the QOrthoNN hidden layers and increase the model's expressive capacity, each QOrthoNN ends with an unconstrained classical linear layer. These quantum transformations are built from the Reconfigurable Beam Splitter (RBS) gate:
	\[
	U_\mathrm{RBS}(\phi) =
	\begin{pmatrix}
		1 & 0 & 0 & 0 \\
		0 & \cos\phi & \sin\phi & 0 \\
		0 & -\sin\phi & \cos\phi & 0 \\
		0 & 0 & 0 & 1
	\end{pmatrix}
	\]
    where $\phi$ denotes the rotation angle of a single RBS gate. The RBS gate rotates only the unary amplitudes associated with $|01\rangle$ and $|10\rangle$, leaving $|00\rangle$ and $|11\rangle$ unchanged. It therefore preserves Hamming weight, so the data register remains within the unary subspace during data loading and the parameterized transformation. Measurements outside this subspace can be discarded during tomography, enabling unary post-selection for error mitigation.

    To describe a single QOrthoNN hidden layer, we consider an $n\times n$ transformation and defer the generalized $m\times n$ formulation to Appendix~\ref{sec:quantum_state_evolution}. QOrthoNNs use unary amplitude encoding, so the encoded input must have unit $L_2$-norm. For the square first-layer construction considered here, the raw input dimension is $d_{\mathrm{in}}=n-1$. We map raw inputs to $[-1,1]$ using fixed training-set min--max bounds, clipping calibration and test values to these bounds. Denoting the normalized vector by $\bar{\mathbf{x}}\in\mathbb{R}^{d_{\mathrm{in}}}$, we scale it by $1/\sqrt{d_{\mathrm{in}}}$ and append the slack coordinate $\sqrt{1-\lVert\bar{\mathbf{x}}/\sqrt{d_{\mathrm{in}}}\rVert_2^2}$, producing the $n$-dimensional unit input. Before each subsequent hidden layer, intermediate outputs are divided by their $L_2$ norm.

    Once the $n$-dimensional normalized input vector $\mathbf{x} \in \mathbb{R}^n$ is prepared, the PQC executes the orthogonal linear transformation $\mathbf{W} \in \mathbb{R}^{n \times n}$ using $n$ data qubits and one ancillary qubit. First, a data-dependent unitary $S(\mathbf{x})$ loads the input vector onto the data register via amplitude encoding, yielding the quantum state:
    \[
        |\psi\rangle_\mathrm{data} = x_1 |10\cdots0\rangle + x_2 |01\cdots0\rangle + \cdots + x_n |00\cdots1\rangle = \sum_{i=1}^{n} x_i |\mathbf{e}_i\rangle
    \]
    where $|\mathbf{e}_i\rangle$ denotes a unary computational basis state. Let $\boldsymbol{\theta}$ denote the vector of RBS rotation angles in a QOrthoNN layer. The corresponding pyramidal unitary $U(\boldsymbol{\theta})$ emulates the weight matrix $\mathbf{W}$ and evolves the data register to:
    \[
        |\psi\rangle = \sum_{j=1}^{n}\sum_{i=1}^{n}W_{ji}x_i|\mathbf{e}_j\rangle
    \]
    Tomography then recovers the transformed amplitudes as classical values using the ancilla and two additional $S$ data loaders; Appendix~\ref{sec:quantum_state_evolution} details this sub-circuit and the complete state evolution. A learned bias and nonlinear activation complete the hidden-layer computation.

    Having established the forward pass, we train the network classically rather than directly on Noisy Intermediate-Scale Quantum (NISQ) hardware \citep{bharti2022nisq,preskill2018nisq}. Each parameterized quantum transformation has an exact classical representation as an orthogonal matrix $\mathbf{W}$ parameterized by the RBS angles $\boldsymbol{\theta}$, allowing standard deep learning frameworks to emulate it without quantum state simulation. Automatic differentiation optimizes these angles at $\mathcal{O}(n^2)$ cost per layer while maintaining orthogonality by construction \citep{kerenidis2022qorthonn,landman2022medicalqorthonn}. This avoids the $\mathcal{O}(n^3)$ Singular Value Decomposition or Stiefel manifold projection steps incurred by projection-based approaches \citep{kerenidis2022qorthonn,landman2022medicalqorthonn,li2021orthonn}. After training, the learned angles are deployed to the PQC with the layer-level depth scaling described above.
    
    For the square transformations analysed above, orthogonal parameterization also controls conditioning. For any activation $\mathbf{x}$ and back-propagated gradient $\mathbf{g}$, an orthogonal weight matrix $\mathbf{W}$ satisfies $\lVert\mathbf{W}\mathbf{x}\rVert_2=\lVert\mathbf{x}\rVert_2$ and $\lVert\mathbf{W}^{\top}\mathbf{g}\rVert_2=\lVert\mathbf{g}\rVert_2$. Thus, a square orthogonal transformation has condition number one and does not itself amplify or attenuate activations or gradients. Prior work also associates orthogonal weights with improved generalization and training stability, although these empirical benefits depend on the architecture and task \citep{li2021orthonn,kerenidis2022qorthonn,landman2022medicalqorthonn}.

    \subsection{Ensemble-Based Model Uncertainty}
	\label{sec:ensemble_construction}
	
	Recent work explores using stochastic quantum measurements to quantify uncertainty \citep{park2023qcp}. For a single QOrthoNN output component $o_j$, however, the shot-induced standard deviation scales as $\operatorname{Std}(o_j)\propto 1/\sqrt{N_\mathrm{shots}}$ and therefore vanishes as the measurement budget increases, regardless of input difficulty or out-of-distribution shifts \citep{xiao2025qdon}. Because this variability does not represent epistemic uncertainty, we instead construct ensembles of independently trained Quantum DeepONets.
	
    To induce predictive diversity, we train each ensemble member, $\mathcal{G}_{\theta_\ell}$, independently on the full dataset using distinct random shuffling and parameter initializations. At inference, the ensemble mean and standard deviation are
    \[
    \mu(\mathbf{u})(\mathbf{y})=\frac{1}{L}\sum_{\ell=1}^{L}\mathcal{G}_{\theta_\ell}(\mathbf{u})(\mathbf{y}),\qquad \sigma(\mathbf{u})(\mathbf{y})=\sqrt{\frac{1}{L}\sum_{\ell=1}^{L}\left(\mathcal{G}_{\theta_\ell}(\mathbf{u})(\mathbf{y})-\mu(\mathbf{u})(\mathbf{y})\right)^2}.
    \]
    Although we evaluated nonparametric bootstrapping (bagging) as an alternative mechanism, we observed no performance gain. This agrees with prior observations for deep ensembles: a bootstrap sample contains only about 63\% unique training examples, and the resulting loss of data can outweigh the added diversity when random initialization and shuffling already provide model variation \citep{lakshminarayanan2017deep,lee2015mheads}.

	\subsection{Exploratory Ensemble Execution}
	
	Executing the ensemble introduces a separate resource trade-off. Sequential evaluation multiplies evaluation time by $L$, whereas naive parallel evaluation requires $\mathcal{O}(Ln)$ qubits. We therefore consider two strategies for managing these costs: a hybrid classical--quantum partition and SPQC-inspired multiplexed execution.

    \subsubsection{Hybrid Classical--Quantum Architectures}

    The DeepONet branch--trunk decomposition permits either subnetwork to be selectively classical. We focus on a classical branch and quantum trunk because the branch representation is computed once per input function, whereas the trunk is evaluated at each of the $M$ query coordinates. Under the idealized model above, a width-$n$ classical branch hidden layer has running time $\mathcal{O}(n^2/P+\log n)$, which becomes $\mathcal{O}(n^2/(PM)+\log(n)/M)$ per query after amortization. Each width-$n$ QOrthoNN trunk hidden layer has depth--shot cost $\mathcal{O}(n\log(n)/\delta^2)$ at every query \citep{kerenidis2022qorthonn,landman2022medicalqorthonn}. The amortized branch work term is lower order when $PM=\omega(n\delta^2/\log(n))$; for fixed $\delta$, this becomes $PM=\omega(n/\log(n))$.
    
    Alternatively, one could reduce resource demands by sharing a single subnetwork across all ensemble members. However, this strategy sacrifices predictive diversity. By forcing the independent, unshared subnetworks to adapt to a fixed intermediate output, this shared representation acts as an informational bottleneck. Thus, we expect it to constrain the ensemble's variance and degrade uncertainty quantification.

    \subsubsection{SPQC-Inspired Multiplexed Ensemble Execution}

	The second strategy uses address-conditioned multiplexing as an intermediate point between sequential and fully parallel execution. Our construction is inspired by Superposed Parameterized Quantum Circuits (SPQCs), which combine Flip-Flop Quantum Random Access Memory (FFQRAM) parameter routing with repeat-until-success (RUS) nonlinearities \citep{patapovich2025spqc}. Patapovich et al. report that FFQRAM state preparation requires $\mathcal{O}(nL)$ one- and two-qubit gates but has $\mathcal{O}(\log n)$ non-Clifford $T$-depth under RUS synthesis and ancilla recycling. We adopt only the address-conditioned routing idea and do not claim these bounds or equivalence to the complete SPQC architecture.
        
    Prepare a $\lceil\log_2 L\rceil$-qubit address register in a uniform superposition over the states $|\ell-1\rangle_{\mathrm{addr}}$, $\ell=1,\ldots,L$, with unused basis states assigned zero amplitude. Uniformly controlled rotations apply the member-specific unary data loader $S(\mathbf{x}^{(\ell)})$ and RBS transformation $U(\boldsymbol{\theta}^{(\ell)})$ to a shared $n$-qubit data register. Including one tomography ancilla, a square width-$n$ layer uses $n+1+\lceil\log_2 L\rceil$ qubits, compared with $n+1$ qubits under sequential execution and $L(n+1)$ qubits under fully parallel execution. Omitting the tomography ancilla, the multiplexed layer produces the state
	\[
	|\psi\rangle=\frac{1}{\sqrt{L}}\sum_{\ell=1}^{L}|\ell-1\rangle_{\mathrm{addr}}\otimes U(\boldsymbol{\theta}^{(\ell)})S(\mathbf{x}^{(\ell)})|\mathbf{e}_1\rangle_{\mathrm{data}}.
	\]
    
    Here, $\mathbf{x}^{(\ell)}$ and $\boldsymbol{\theta}^{(\ell)}$ are the layer input and RBS parameters for ensemble member $\ell$, and $|\mathbf{e}_1\rangle_{\mathrm{data}}$ is the initial unary basis state. The members share the external input in the first layer, whereas their inputs can differ in subsequent layers after member-specific activations. Tomography extracts each member's output from outcomes associated with its address state; the bias, nonlinear activation, and processing between quantum layers are then applied classically.
    
	Compared with a single-member circuit, multiplexing increases circuit depth through controlled routing and requires $\mathcal{O}(L)$ times as many shots to maintain fixed per-member precision. Relative to sequential execution, it therefore retains the same $\mathcal{O}(L)$ total-shot scaling and uses $\lceil\log_2 L\rceil$ additional qubits, but consolidates $L$ member-specific circuit instances into one combined circuit per layer input. \citet{patapovich2025spqc} contrast the complete SPQC with sequential execution, describing it as trading ``the classical overhead of repeated circuit initialization for a small ancilla register.'' Because our implementation adopts only the routing idea and uses address-resolved tomography and classical processing between layers, we do not assume that this runtime claim transfers to our construction. Relative to fully parallel execution, the construction reduces qubit use from $L(n+1)$ to $n+1+\lceil\log_2 L\rceil$. We therefore evaluate it as a compact proof-of-concept with $L=4$ rather than provide a complete characterization of the qubit--depth--compilation--sampling trade-off.

    \section{Conformal Calibration}
    \label{sec:uq}

    The ensemble standard deviation measures disagreement among independently trained models, but does not by itself provide calibrated prediction intervals. We therefore use a separate calibration dataset to convert this disagreement into uncertainty bounds \citep{angelopoulous2023conformal,vovk1999conformal,vovk2022conformallearning,shafer2008conformaltutorial,papadopoulos2002conformalregression}.
	
	Conventional split conformal prediction gives a finite-sample marginal coverage guarantee at a new test point under exchangeability of the calibration and test points \citep{vovk1999conformal,papadopoulos2002conformalregression,angelopoulous2023conformal}. Applied naively to DeepONet, each query evaluation would be treated as a separate example and assigned a pointwise prediction interval. However, the query evaluations associated with one input function jointly represent a discretized output function and may be strongly dependent. We therefore take one input function together with its query coordinates and output values as the unit of calibration, and assume exchangeability only across these complete units. Nothing is assumed about the queries within a unit; they may be arbitrarily dependent, as evaluations of a single smooth solution often are. Consequently, the guarantee is neither fixed-query nor simultaneous. It constrains coverage averaged over the query coordinates of a new output function.

	For any generic input function $\mathbf{u}$ and query coordinate $\mathbf{y}$, let $\mu(\mathbf{u})(\mathbf{y})$ and $\sigma(\mathbf{u})(\mathbf{y})$ denote the ensemble mean and standard deviation of Section~\ref{sec:ensemble_construction}, and set $\tilde\sigma(\mathbf{u})(\mathbf{y})=\sigma(\mathbf{u})(\mathbf{y})+\epsilon$ where $\epsilon>0$. Abbreviate $\mu_{ij}=\mu(\mathbf{u}_{i})(\mathbf{y}_{ij})$ and $\tilde\sigma_{ij}=\tilde\sigma(\mathbf{u}_{i})(\mathbf{y}_{ij})$, and define the nonconformity score
	\[
	R_{ij}=\frac{|s_{ij}-\mu_{ij}|}{\tilde\sigma_{ij}}.
	\]
	Normalizing by the ensemble standard deviation makes the interval width vary across the domain, widening where the ensemble members disagree.
	
	Let the calibration dataset contain $K$ input functions,
	\[
	\mathcal{D}_{\mathrm{cal}}
	=
	\{\mathcal{D}_{\mathrm{cal},i}\}_{i=1}^{K},
	\qquad
	\mathcal{D}_{\mathrm{cal},i}
	=
	\{(\mathbf{u}_{i},\mathbf{y}_{ij},s_{ij})\}_{j=1}^{M_i},
	\]
	where input function $i$ is evaluated at $M_i$ query coordinates. Let $N_{\mathrm{cal}}=\sum_{i=1}^{K}M_i$ denote the total number of calibration scores, and let $R^{\mathrm{cal}}_{(1)}\leq\cdots\leq R^{\mathrm{cal}}_{(N_{\mathrm{cal}})}$ denote those scores in increasing order. We take $\widehat q$ to be the linearly interpolated empirical $(1-\alpha)$-quantile: setting
	\[
	a=(N_{\mathrm{cal}}-1)(1-\alpha),
	\qquad
	r=\lfloor a\rfloor+1,
	\qquad
	\gamma=a-\lfloor a\rfloor,
	\]
	we define the threshold as
	\[
	\widehat q=(1-\gamma)R^{\mathrm{cal}}_{(r)}+\gamma R^{\mathrm{cal}}_{(r+1)}.
	\]
	We use this common interpolation convention throughout. Other empirical-quantile conventions can yield different, and potentially stronger, finite-sample bounds. The prediction interval is then
	\[
	C_\alpha(\mathbf{u},\mathbf{y})
	=
	\left[
	\mu(\mathbf{u})(\mathbf{y})-\widehat q\,\tilde\sigma(\mathbf{u})(\mathbf{y}),
	\;
	\mu(\mathbf{u})(\mathbf{y})+\widehat q\,\tilde\sigma(\mathbf{u})(\mathbf{y})
	\right].
	\]
	The true value lies in $C_\alpha(\mathbf{u},\mathbf{y})$ if and only if the score at that coordinate is at most $\widehat q$. Section~\ref{sec:guarantee} uses this equivalence to lower-bound the query-averaged coverage of a new output function.
	
    \subsection{Finite-Sample Guarantee}
    \label{sec:guarantee}
    
    We now establish a finite-sample lower bound on the query-averaged coverage defined in Section~\ref{sec:problemformulation}. Let $\mathbf{R}_{i}=(R_{i1},\ldots,R_{iM})$ denote the score vector of unit $i$. The proof adapts the group-symmetrization argument of \citet{lee2023hierarchical} to query-averaged coverage without within-unit exchangeability.
    
    \begin{assumption}
    	\label{ass:exch}
    	Every input function in the calibration and test sets is evaluated at the same number $M$ of query coordinates, so that $N_{\mathrm{cal}}=KM$. The ensemble is trained on data disjoint from the calibration and test sets. Conditional on the trained ensemble and a fixed inference mechanism, all calibration and test units are jointly exchangeable. Nothing is assumed about the joint distribution of the $M$ queries within a unit.
    \end{assumption}
    
    The score vectors are therefore exchangeable because the same prediction and scoring procedures are applied to every unit.
	    
	\begin{theorem}
		\label{thm:coverage}
		Let $\mathbf{u}_{*}$ denote the input function of a generic test unit, evaluated at query coordinates $\mathbf{y}_{*,1},\ldots,\mathbf{y}_{*,M}$ with ground truth values $s_{*,1},\ldots,s_{*,M}$, and let $r=\lfloor(KM-1)(1-\alpha)\rfloor+1$. Under Assumption~\ref{ass:exch},
		\[
		\mathbb{E}\left[
		\frac{1}{M}\sum_{j=1}^{M}
		\mathbf{1}\left\{s_{*,j}\in 	C_\alpha(\mathbf{u}_{*},\mathbf{y}_{*,j})\right\}
		\right]
		\geq
		\frac{r}{(K+1)M}.
		\]
	\end{theorem}
	
	The expectation is marginal over the calibration and test units, conditional on the trained ensemble and the fixed inference mechanism. The quantity inside the expectation is the fraction of covered query coordinates for the test unit.
	
	\begin{proof}
		Index the $K$ calibration units and the test unit by $i\in\{1,\ldots,K+1\}$, with $i=K+1$ denoting the generic test unit, so that $R_{K+1,j}=R_{*,j}$. For the proof, pool their $(K+1)M$ scores and let $R_{(k)}$ denote the $k$-th smallest score in this augmented pool.
		\[
		R_{(1)}\leq\cdots\leq R_{((K+1)M)}.
		\]
		For each $i$, let $\widehat q_{-i}$ denote the leave-one-unit-out threshold obtained after removing the complete score vector $\mathbf{R}_i$ from this pool. When $i=K+1$, this is the calibration threshold $\widehat q$.
		
		\emph{Claim 1: each leave-one-unit-out threshold dominates the pooled $r$-th score.}
		
		Let
		\[
		R^{-i}_{(1)}\leq\cdots\leq R^{-i}_{(KM)}
		\]
		denote the remaining scores after removing $\mathbf{R}_i$. Because $\widehat q_{-i}$ interpolates between the $r$-th and $(r+1)$-th smallest remaining scores,
		\[
		\widehat q_{-i}\geq R^{-i}_{(r)}.
		\]
		Removing scores from the total $(K+1)M$ pool cannot decrease the $r$-th smallest value. Therefore,
		\[
		R^{-i}_{(r)}\geq R_{(r)},
		\]
		and hence
		\[
		\widehat q_{-i}\geq R^{-i}_{(r)}\geq R_{(r)}.
		\]
		
		\emph{Claim 2: the leave-one-unit-out covered fractions satisfy an average lower bound.}
		
		Define
		\[
		V_i=\frac{1}{M}\sum_{j=1}^{M}\mathbf{1}\{R_{ij}\leq\widehat q_{-i}\}
		\]
		for the fraction of unit $i$'s queries covered by its own threshold. The $r$ smallest scores in the total pool are all at most $R_{(r)}$, thus
		\[
		\frac{1}{K+1}\sum_{i=1}^{K+1}\frac{1}{M}\sum_{j=1}^{M}\mathbf{1}\{R_{ij}\leq R_{(r)}\}\geq\frac{r}{(K+1)M}
		\]
		Since $\widehat q_{-i} \geq R_{(r)}$ for every $i$ by Claim 1, we have
		\[
		\frac{1}{K+1}\sum_{i=1}^{K+1}\frac{1}{M}\sum_{j=1}^{M}\mathbf{1}\{R_{ij}\leq \widehat q_{-i}\}\geq\frac{r}{(K+1)M}
		\]
		and
		\[
		\frac{1}{K+1}\sum_{i=1}^{K+1}V_i\geq\frac{r}{(K+1)M}.
		\]
		
		\emph{Conclusion: exchangeability transfers the average bound to the test unit.}
		
		Define the mapping
		\[
		F(\mathbf{R}_1,\ldots,\mathbf{R}_{K+1})=(V_1,\ldots,V_{K+1}).
		\]
		For any permutation $\pi$ of the units,
		\[
		F(\mathbf{R}_{\pi(1)},\ldots,\mathbf{R}_{\pi(K+1)})=(V_{\pi(1)},\ldots,V_{\pi(K+1)}).
		\]
		This holds because each $V_i$ treats $\mathbf{R}_i$ as the held-out unit and uses the remaining score vectors without regard to their order.
		
		Because the score vectors are exchangeable,
		\[
		(\mathbf{R}_{\pi(1)},\ldots,\mathbf{R}_{\pi(K+1)})\overset{d}{=}(\mathbf{R}_1,\ldots,\mathbf{R}_{K+1}).
		\]
		Applying $F$ to both sides gives
		\[
		(V_{\pi(1)},\ldots,V_{\pi(K+1)})\overset{d}{=}(V_1,\ldots,V_{K+1}).
		\]
		Thus, the $V_i$'s are exchangeable. In particular,
		\[
		\mathbb{E}[V_1]=\cdots=\mathbb{E}[V_{K+1}].
		\]
		Taking expectations in Claim 2 gives
		\[
		\frac{1}{K+1}\sum_{i=1}^{K+1}\mathbb{E}[V_i]\geq\frac{r}{(K+1)M}.
		\]
		Because all terms in the sum have the same expectation,
		\[
		\frac{1}{K+1}\sum_{i=1}^{K+1}\mathbb{E}[V_i]=\frac{1}{K+1}\sum_{i=1}^{K+1}\mathbb{E}[V_{K+1}]=\mathbb{E}[V_{K+1}].
		\]
		Therefore,
		\[
		\mathbb{E}[V_{K+1}]\geq\frac{r}{(K+1)M}.
		\]
		By definition, and because $\widehat q_{-(K+1)}=\widehat q$,
		\[
		V_{K+1}=\frac{1}{M}\sum_{j=1}^{M}\mathbf{1}\{R_{K+1,j}\leq\widehat q_{-(K+1)}\}=\frac{1}{M}\sum_{j=1}^{M}\mathbf{1}\{R_{K+1,j}\leq\widehat q\}.
		\]
		The equivalence between score thresholding and interval coverage gives
		\[
		V_{K+1}=\frac{1}{M}\sum_{j=1}^{M}\mathbf{1}\left\{s_{*,j}\in C_\alpha(\mathbf{u}_{*},\mathbf{y}_{*,j})\right\}.
		\]
		Hence,
		\[
		\mathbb{E}\left[\frac{1}{M}\sum_{j=1}^{M}\mathbf{1}\left\{s_{*,j}\in C_\alpha(\mathbf{u}_{*},\mathbf{y}_{*,j})\right\}\right]=\mathbb{E}[V_{K+1}]\geq\frac{r}{(K+1)M},
		\]
		which proves the result.

		\end{proof}

	An equivalent probability interpretation is obtained by drawing $J$ uniformly from $\{1,\ldots,M\}$, independently of the calibration and test units. Conditional on the trained ensemble and the fixed inference mechanism,
	\[
	\mathbb{P}\left\{s_{*,J}\in C_\alpha(\mathbf{u}_*,\mathbf{y}_{*,J})\right\}=\mathbb{E}\left[\frac{1}{M}\sum_{j=1}^{M}\mathbf{1}\left\{s_{*,j}\in C_\alpha(\mathbf{u}_*,\mathbf{y}_{*,j})\right\}\right]\geq\frac{r}{(K+1)M}.
	\]
	
	To relate this bound to $1-\alpha$, write
	\[
	r=(KM-1)(1-\alpha)+\eta,\qquad 0<\eta\leq 1,
	\]
	where $\eta$ accounts for rounding. Then
	\[
	\frac{r}{(K+1)M}=(1-\alpha)\frac{K}{K+1}+\frac{\eta-(1-\alpha)}{(K+1)M}.
	\]
	Therefore,
	\[
	\lim_{K\rightarrow\infty}\frac{r}{(K+1)M}=1-\alpha.
	\]
	The difference from $1-\alpha$ is $\mathcal{O}(1/K)$, so the bound can already be close to the target for practical calibration-set sizes. Increasing $M$ alone does not remove the finite-unit gap: for fixed $K$, the bound approaches $(1-\alpha)K/(K+1)$ as $M\rightarrow\infty$.
	
	The guarantee applies to the expectation, or equivalently to the probability for a uniformly selected query. The covered fraction for a particular calibration and test realization may fall below the bound. The theorem also provides no nontrivial upper bound; the marginal coverage probability may be as large as one.

    \section{Experiments}
    \label{sec:experiments}

    \subsection{Setup and Evaluation Metrics}

	We evaluate predictive accuracy using relative $L_2$ error and uncertainty calibration using empirical query-averaged coverage, average interval width, and maximum interval width. For $N_{\mathrm{test}}$ test input--output function pairs, each evaluated at $M$ query coordinates, empirical query-averaged coverage is
	\[
	\widehat{\mathrm{Cov}}=\frac{1}{N_{\mathrm{test}}}\sum_{i=1}^{N_{\mathrm{test}}}\frac{1}{M}\sum_{j=1}^{M}\mathbf{1}\left\{s_{ij}\in C_\alpha(\mathbf{u}_i,\mathbf{y}_{ij})\right\}.
	\]
	For each test pair, we compute the relative $L_2$ error of the ensemble mean prediction across its $M$ query coordinates and then average over all test pairs:
	\[
	\frac{1}{N_{\mathrm{test}}}\sum_{i=1}^{N_{\mathrm{test}}}
	\frac{\left(\sum_{j=1}^{M}\left(\mu(\mathbf{u}_i)(\mathbf{y}_{ij})-s_{ij}\right)^2\right)^{1/2}}
	{\left(\sum_{j=1}^{M}s_{ij}^{2}\right)^{1/2}}.
	\]

	Across all experiments, the branch and trunk QOrthoNNs use the same hidden width and depth. The first branch layer projects the $(d_u+1)$-dimensional slack-augmented input down to the reported width, whereas the first trunk layer projects its lower-dimensional query representation up to that width.

    Direct state-vector simulation becomes impractical for the larger circuits considered here. For ideal evaluations, we therefore use the mathematically equivalent classical Orthogonal Neural Network, which reproduces the noiseless QOrthoNN transformation without simulating the full quantum state. Because of this exact equivalence, we refer to these direct classical evaluations as ideal simulations.

    Beyond these ideal baselines, we evaluate the framework under varying degrees of quantum noise. To conduct hyperparameter ablations efficiently, we first use a gate-based depolarizing model with finite-shot sampling. For single- and two-qubit gates acting on a $k$-qubit subsystem $A$ (where $k \in \{1, 2\}$), the local depolarizing error on the quantum state $\rho$ is modelled as:
    \[\mathcal{E}_A(\rho) = (1-\lambda)\rho + \lambda \left( \frac{I_A}{2^k} \otimes \operatorname{Tr}_A[\rho] \right)\]
    with noise strength $\lambda$ scaled by 0.8 for two-qubit operations to equalize error rates across 1 and 2-qubit depolarizing noise channels. We evaluate circuits under this model by transpiling them to IBM's Eagle basis gate set $\{\mathrm{ECR}, \mathrm{R}_z, \mathrm{SX}, \mathrm{X}, \mathrm{I}\}$ \citep{abughanem2025eagle, xiao2025qdon}, assuming nearest-neighbour connectivity. For circuit-depth comparisons, we also transpile the circuits to IBM's current Heron basis gate set $\{\mathrm{CZ},\mathrm{I},\mathrm{R}_x,\mathrm{R}_z,\mathrm{R}_{zz},\mathrm{SX},\mathrm{X}\}$ and report depths for both architectures \citep{IBMQuantumComputers}. To accelerate simulation, we apply multinomial sampling to the evolved density matrix; Appendix~\ref{sec:multinomial} shows close agreement with full per-shot simulation. We further evaluate selected compact circuits using device-calibrated Qiskit Aer noise models that combine readout error, depolarizing error, and thermal relaxation \citep{qiskit2024}. For this validation, circuits are transpiled according to the reported calibration parameters, basis gates, and topological coupling maps of three representative quantum processing units: IBM Brisbane (127-qubit Eagle r3), IBM Torino (133-qubit Heron r1), and IBM Marrakesh (156-qubit Heron r2) \citep{IBMQuantumComputers}. The snapshots of calibration data used in our experiments can be found in Appendix~\ref{sec:snapshots}.

	In these noisy simulations, we use unary-subspace post-selection to mitigate quantum noise. Because QOrthoNN data are encoded in the Hamming-weight-one subspace, we discard tomography outcomes whose data-register bits have any other Hamming weight. This removes errors that leave the unary subspace, at the cost of discarding shots.

	Dataset split sizes and training hyperparameters for all experiments are reported in Appendix~\ref{sec:train_hp}.

    \subsection{Synthetic Benchmarks}

    \begin{table}[t]
        \caption{Summary of model performance across synthetic experiments under ideal simulation with 90\% target coverage. The metrics from left to right are Relative $L_2$ Error (\%), Empirical Coverage (\%), Average Width, and Maximum Interval Width.}
        \label{tab:results_summary_syn}
        \begin{center}
            \setlength{\tabcolsep}{5pt}
            \begin{tabular}{lcccc}
                \hline
                
                \textbf{Experiment} & \textbf{Rel. $L_2$ Error (\%)} & \textbf{Coverage (\%)} & \textbf{Avg. Width} & \textbf{Max. Width} \\
                \hline

                Antiderivative (Ens. Size = 4) & 0.46 & 88.40 & 0.004 & 0.044 \\
                Antiderivative (Ens. Size = 8) & 0.46 & 92.13 & 0.005 & 0.080 \\
                Advection (Ens. Size = 4)      & 2.38 & 89.36 & 0.062 & 0.751 \\
                Advection (Ens. Size = 8)      & 2.28 & 88.99 & 0.053 & 0.621 \\

                \hline
            \end{tabular}
        \end{center}
    \end{table}

    \begin{table}[t]
        \caption{Architecture and hardware specifications for the quantum DeepONet ensembles for the synthetic benchmarks. Width denotes the number of neurons per layer. Params and circuit depth report the maximum per-layer values for total RBS parameters and transpiled PQC depth (on IBM Heron/Eagle), respectively. All models use the SiLU activation function. (Config: R = ResNet, F = Fourier features). The number of qubits used for a single model is $d_u + 2$ (accounting for the $d_u$ raw inputs, one appended slack coordinate, and one measurement ancilla).}
        \label{tab:model_architectures_syn}
        \begin{center}
            \begin{tabular}{lccccccccc}
                \hline
                \textbf{Experiment} & $d_u$ & $d_y$ & \textbf{Layers} & \textbf{Width} & \textbf{Params} & \textbf{Heron} & \textbf{Eagle} & \textbf{Config} & \textbf{Ens. Size} \\
                \hline

                Antiderivative & 10 & 1 & 2 & 10 & 55  & 165 & 321 & -- & 4/8 \\
                Advection      & 20 & 2 & 7 & 20 & 210 & 305 & 601 & R  & 4/8 \\

                \hline
            \end{tabular}
        \end{center}
    \end{table}

    \begin{figure}[t]
        \begin{center}
            \includegraphics[width=\columnwidth]{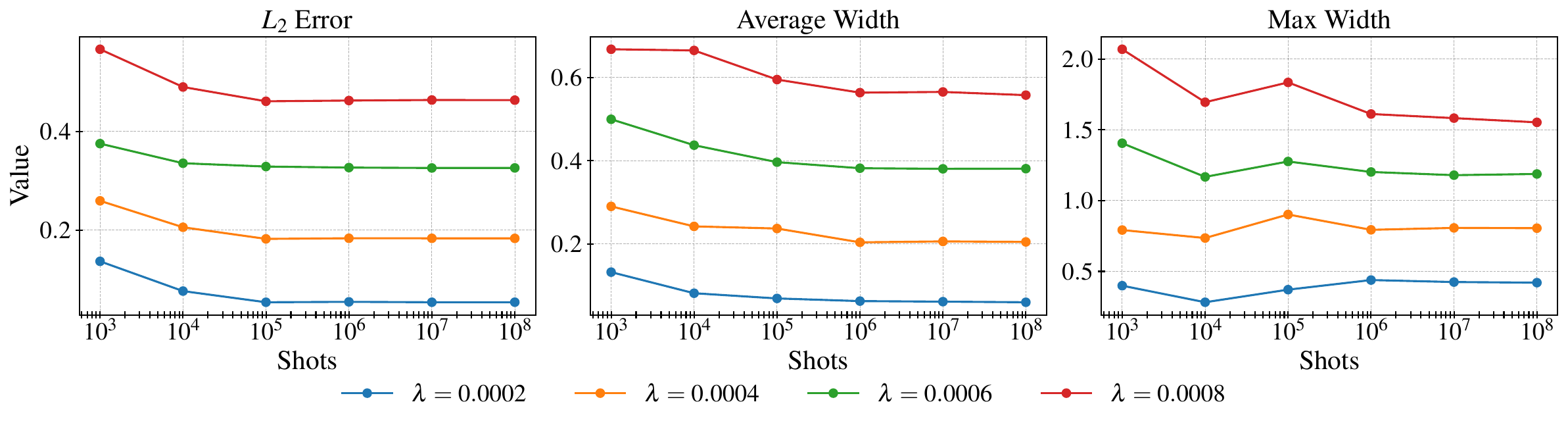}
        \end{center}

        \caption{Effects of depolarizing noise strength $\lambda$ and shot count on antiderivative prediction performance. The panels report relative $L_2$ error (left), average interval width (middle), and maximum interval width (right). Error and interval widths generally decrease with lower noise strength and stabilize as the shot count increases.}

        \label{fig:simplified_anti}
    \end{figure}

    \subsubsection{Antiderivative Operator}

    We first evaluate our framework on a synthetic antiderivative operator $\mathcal{A}$, which maps an input function $u(x)$ on $x\in[0,1]$ to its integral, defined by $\frac{d}{dx}\mathcal{A}(u)(x)=u(x)$ subject to $\mathcal{A}(u)(0)=0$. To generate the dataset, we sample smooth input functions from a Gaussian random field (GRF) prior, $u\sim\mathcal{GP}(0,k_l)$. The covariance kernel is $k_l(x,x')=\exp\left(-\frac{d(x,x')^2}{2l^2}\right)$, where $d(\cdot,\cdot)$ denotes the distance between input coordinates and the length scale $l$ controls the smoothness of the sampled functions.

    For this task, we deploy an ensemble of $L=8$ independent models, using a 2-layer, 10-neuron QOrthoNN architecture. As detailed in Table~\ref{tab:model_architectures_syn}, this model requires 55 RBS rotation parameters in its largest quantum layer and transpiles to a maximum circuit depth of 165 on the IBM Heron topology and 321 on IBM Eagle. Under ideal noiseless simulation, this quantum ensemble achieves a 0.46\% relative $L_2$ error and 92.13\% empirical coverage for a target rate of $1-\alpha=0.90$.

	As an illustrative calculation, the antiderivative experiment has $K=50$, $M=30$, and $\alpha=0.1$, for which Theorem~\ref{thm:coverage} gives the finite-sample lower bound $r/((K+1)M)=88.24\%$. The same calculation applies to the other experiments using their respective values of $K$, $M$, and $\alpha$.

    We further studied the effects of depolarizing noise on the approximation of the antiderivative operator with $L=8$ ensemble members each composed of the same QOrthoNN architecture described above. As shown in Figure~\ref{fig:simplified_anti}, prediction errors and interval widths generally decrease as the depolarizing-noise strength decreases and the shot count increases.

    \subsubsection{Advection Equation}

    To evaluate the framework on a dynamical system, we next consider the 1D advection equation, $\frac{\partial u}{\partial x} + \frac{\partial u}{\partial t} = 0$ for $x, t \in [0, 1]$, subject to periodic boundary conditions. Here, the operator learns the solution mapping from the initial condition $u(x,0)$ to the full spatiotemporal solution field $u(x,t)$. To enforce periodicity, we sample the initial conditions from a Gaussian random field with a unit-period Exp-Sine-Squared kernel, $k_l(x,x')=\exp\left(-2\sin^2(\pi d(x,x'))/l^2\right)$.

    Because advection demands tracking features across spacetime, we deploy a deeper, 7-layer QOrthoNN architecture featuring residual connections to stabilize gradient flow. This model processes 20 uniformly sampled spatial locations for the branch input and a Cartesian product of 50 spatial and 50 temporal coordinates, giving 2,500 query points. Consequently, this expanded architecture requires 210 RBS rotation parameters in its largest quantum layer, transpiling to a maximum circuit depth of 305 on IBM Heron and 601 on IBM Eagle. Under ideal simulation, the 8-model advection ensemble achieves a 2.28\% relative $L_2$ error and 88.99\% empirical coverage for a target rate of $1-\alpha=0.90$. With $K=200$, $M=2500$, and $\alpha=0.1$, Theorem~\ref{thm:coverage} gives a finite-sample lower bound of $89.55\%$ for this experiment. Because the theorem bounds an expectation rather than a single calibration and test realization, the coverage reported here falls slightly below that value. Table~\ref{tab:results_summary_syn} summarizes the predictive accuracy and coverage of both ensemble sizes. We next examine how these metrics change under device-calibrated composite noise simulations.

    \subsection{Empirical Coverage Under Device-Calibrated Noise Models}

    To evaluate robustness under device-calibrated composite noise simulations, we measure coverage for the antiderivative ensemble under noise models calibrated to IBM Marrakesh, IBM Torino, and IBM Brisbane. As shown in Figure~\ref{fig:cov_guar}, the observed coverage remains near or above the nominal $1-\alpha=0.90$ target across measurement budgets from 5,000 to 100,000 shots. These results are consistent with stable calibration and test conditions, but empirical coverage alone cannot verify exchangeability. The finite-sample result in Section~\ref{sec:uq} applies when the calibration and test input--output function pairs are exchangeable and their predictions are generated under the same stationary inference and hardware-noise mechanism.

	\begin{figure}[t]
		\begin{center}
			\includegraphics[width=\columnwidth]{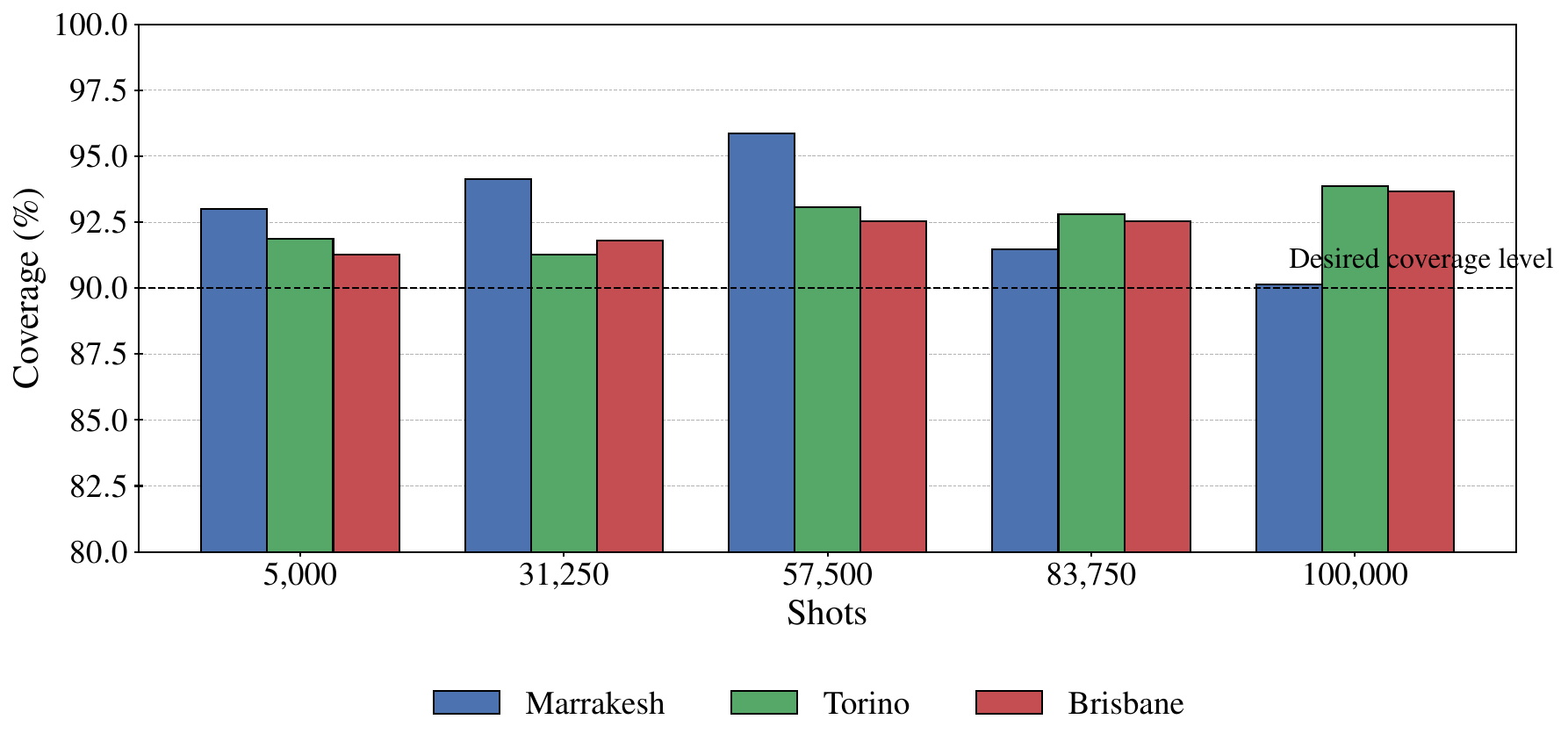}
		\end{center}
		
		\caption{Coverage on the antiderivative test set under noise models calibrated to IBM Marrakesh, IBM Torino, and IBM Brisbane. The dashed line denotes the nominal 90\% target.}
		
		\label{fig:cov_guar}
	\end{figure}

    \subsection{Real Power-Engineering Datasets}

    In power-system applications, distinct operational settings motivate three operator classes. Each input function is a measured system trajectory, and each output function is the corresponding response trajectory. We therefore use ``input trajectory'' and ``output trajectory'' in place of input and output functions throughout this section and the following experimental subsections.
    
    The first is the Transient-to-Transient Operator ($\mathcal{F}$), which maps an entire temporal input trajectory directly to an entire output trajectory. Because this global mapping evaluates the complete dynamic evolution simultaneously, it is essential for offline applications such as post-mortem event analysis and long-term system planning \citep{kwatny2016psdynamics, kundur2004psdynamics}. Formally, this operator functions across the Hilbert spaces $\mathcal{U} = L^2([0,T]; \mathbb{R}^d)$ and $\mathcal{Y} = L^2([0,T]; \mathbb{R}^d)$, where $d$ represents the state-space dimensionality. Given a dataset of observed input--output trajectory pairs, $\{(\mathbf{u}_i, \mathbf{s}_i)\}_{i=1}^N$ where $\mathbf{s}_i = \mathcal{F}[\mathbf{u}_i]$, the objective is to learn a quantum surrogate, $\hat{\mathcal{F}}$.

    The second class is the Causal Forecasting Operator $(\mathcal H)$, which supports proactive grid management strategies such as long-range dynamic stability assessments. Unlike the global mapping of the transient operator, $\mathcal H$ explicitly enforces causality by mapping past pre-event histories to future post-event predictions. Formally defined as $\mathcal H: \mathcal U_{\mathrm{past}} \to \mathcal Y_{\mathrm{future}}$, this operator processes historical input trajectories from the space $\mathcal{U}_{\mathrm{past}} = L^2([t_0 - \tau, t_0]; \mathbb{R}^d)$ to predict the system's future dynamic evolution within the space $\mathcal{Y}_{\mathrm{future}} = L^2([t_0, t_0 + T_f]; \mathbb{R}^d)$, where $t_0$ denotes the time of a critical grid contingency and $d$ represents the physical state-space dimensionality.

    The final class is the Causal Pointwise Transient Operator $(\mathcal P)$, designed for real-time applications such as autonomous grid monitoring and closed-loop control. Mathematically, this is formulated as a time-indexed family of operators, $\{\mathcal{P}_t\}_{t \in [0,T]}$. Each specific operator, $\mathcal{P}_t : L^2([0,t]; \mathbb{R}^d) \to \mathbb{R}^d$, maps a historical input trajectory restricted to the current time $[0,t]$ to a single, localized state prediction at the immediate future step, $t + \Delta t$. However, because integrating a continuously expanding historical domain from $t=0$ is computationally prohibitive for high-frequency real-time inference, we approximate this mapping using a fixed memory window, $\tau > 0$. We therefore truncate the evaluation to a sliding window: $\mathcal{P}_t[u] \approx \mathcal{P}_t\left[u|_{[t-\tau,t]}\right]$.
    
    Across these three operator classes, we construct prediction sets at the corresponding output coordinates and evaluate their coverage. Both offline tasks augment the scalar temporal query $t$ with Fourier features to mitigate spectral bias from transient dynamics. For each task, we extract the five dominant frequencies from the training outputs using an FFT with a Hanning window and map $t$ to
    \[
    \mathbf{y}(t) = [t, \cos(2\pi f_1 t), \sin(2\pi f_1 t), \dots, \cos(2\pi f_5 t), \sin(2\pi f_5 t)]^{\top} \in \mathbb{R}^{11}.
    \]
    
        We now evaluate the framework on the corresponding power-engineering tasks, with the model architectures summarized in Table~\ref{tab:model_architectures_all} and the resulting accuracy and calibration metrics in Table~\ref{tab:results_summary_all}.

    \begin{table}[t]

        \caption{Summary of model performance across offline and online experiments under ideal simulation. The target coverage for all uncertainty quantification tasks is 90\%.}
        \label{tab:results_summary_all}

        \begin{center}
            \begin{tabular}{lcccc}
                \hline
                \textbf{Experiment} & \textbf{Rel. $L_2$ Error (\%)} & \textbf{Coverage (\%)} & \textbf{Avg. Width} & \textbf{Max. Width} \\
                \hline

                Offline V-to-P & 4.08  & 89.74 & 0.001 & 0.063 \\
                Offline V-to-V & 12.49 & 89.66 & 0.345 & 2.691 \\
                Online V-to-V  & 5.49  & 90.11 & 0.148 & 1.328 \\

                \hline
            \end{tabular}
        \end{center}

    \end{table}

    \begin{table}[t]

        \caption{Architecture and hardware specifications for the quantum DeepONet ensembles across offline and online tasks. Width denotes the number of neurons per layer. Params and circuit depth report the maximum per-layer values for total RBS parameters and transpiled PQC depth (on IBM Heron/Eagle), respectively. All models use the SiLU activation function. (Config: R = ResNet, F = Fourier features). The number of qubits used for a single model is $d_u + 2$ (accounting for the $d_u$ raw inputs, one appended slack coordinate, and one measurement ancilla).}
        \label{tab:model_architectures_all}

        \begin{center}
            \begin{tabular}{lccccccccc}
                \hline
                \textbf{Experiment} & $d_u$ & $d_y$ & \textbf{Layers} & \textbf{Width} & \textbf{Params} & \textbf{Heron} & \textbf{Eagle} & \textbf{Config} & \textbf{Ens. Size} \\
                \hline

                Offline V-to-P & 20 & 11 & 6 & 20 & 210 & 305 & 601 & R, F & 8 \\
                Offline V-to-V & 20 & 11 & 6 & 20 & 210 & 305 & 601 & R, F & 8 \\
                Online V-to-V  & 5  & 1 & 2 & 5  & 15  & 94  & 181 & --   & 4 \\

                \hline
            \end{tabular}
        \end{center}

    \end{table}

    \subsubsection{Offline Voltage-to-Active-Power}

    We first evaluate the Transient-to-Transient Operator $(\mathcal F)$ on an offline cross-domain task, mapping a 10-second voltage trajectory to its corresponding active power trajectory. The voltage input is discretized into a branch vector $\mathbf{u}\in \mathbb R^{20}$, while the active power output is queried across 100 Fourier-augmented trunk coordinates. Because the active-power trajectories contain high-frequency transients, we smooth the targets using a zero-phase forward--backward fourth-order Butterworth low-pass filter with cutoff frequency $f_c=f_s/100$. The network therefore learns the operator mapping the voltage trajectory to the low-pass-filtered active-power trajectory.

    For this task, we deploy an $L=8$ ensemble of 6-layer, 20-neuron QOrthoNNs with residual connections. This configuration uses a maximum of 210 RBS rotation parameters in a quantum layer and transpiles to maximum circuit depths of 305 on IBM Heron and 601 on IBM Eagle. Under ideal simulation, the ensemble achieves a 4.08\% relative $L_2$ error and 89.74\% coverage, close to the nominal target of 90\%.

    \subsubsection{Offline Voltage-to-Voltage}

    We next test the Causal Forecasting Operator $(\mathcal H)$ on an offline voltage prediction task, mapping a 1.9-second pre-fault history to a 2.0-second post-fault trajectory. The input history is discretized into a branch vector $\mathbf{u}\in \mathbb R^{20}$, while the forecast horizon is queried across 100 uniformly spaced, Fourier-augmented trunk coordinates.

    This task uses the same $L=8$ ensemble architecture and circuit resources as the preceding voltage-to-active-power benchmark. Under ideal simulation, the framework achieves a 12.49\% relative $L_2$ error and 89.66\% coverage, close to the nominal target of 90\%.

    \subsubsection{Online Voltage-to-Voltage}

	Our final power-system benchmark evaluates the Causal Pointwise Transient Operator $(\mathcal P)$ on online voltage prediction. Given a fixed window of recent voltage measurements, the operator predicts the voltage at the next time step. 
	
	Unlike the preceding offline tasks, this online task does not consist of one input trajectory evaluated at multiple output coordinates. Instead, each parent trajectory generates a sequence of sliding-window input--output examples. For parent trajectory $i$, $\mathbf{u}_{ij}$ is the input window for example $j$, $\mathbf{y}_{ij}$ is its query coordinate, and $s_{ij}$ is the corresponding next-step target.
	
	Windows from the same parent trajectory may overlap and are therefore strongly dependent. We split the data by complete parent trajectories, so all windows from one trajectory belong to only one of the training, calibration, and test sets. We define the exchangeable unit as the trajectory-level block
	\[
	\mathcal{B}_i=\{(\mathbf{u}_{ij},\mathbf{y}_{ij},s_{ij})\}_{j=1}^{M}.
	\]
	
	For each block, define
	\[
	R_{ij}=\frac{|s_{ij}-\mu(\mathbf{u}_{ij})(\mathbf{y}_{ij})|}{\tilde\sigma(\mathbf{u}_{ij})(\mathbf{y}_{ij})},\qquad \mathbf{R}_i=(R_{i1},\ldots,R_{iM}).
	\]
	We pool the scores from the $K$ calibration blocks and compute $\widehat q$ using the calibration rule in Section~\ref{sec:uq}. Although Theorem~\ref{thm:coverage} is stated using a common input function within each unit, its proof requires only jointly exchangeable score vectors of equal length. Conditional on the trained ensemble and inference mechanism, we assume that the calibration and test blocks are jointly exchangeable. Because every block contains the same number $M$ of windows and the same scoring rule is applied to every block, the proof applies directly without requiring exchangeability of the windows within a block.

	\begin{figure}[t]
		
		\begin{center}
			\includegraphics[width=\columnwidth]{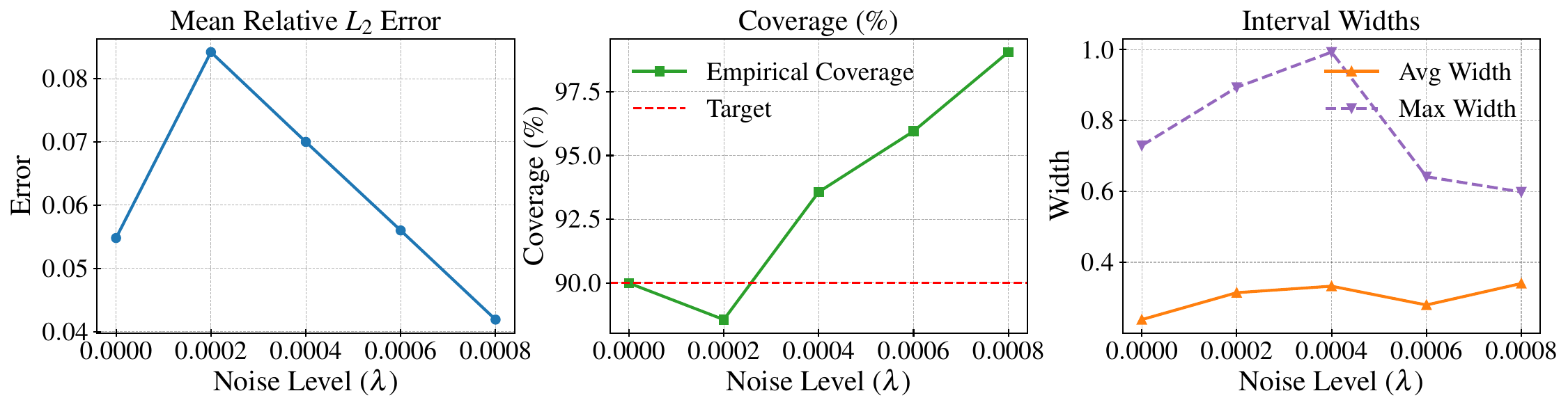}
		\end{center}
		
		\caption{Impact of depolarizing noise strength $\lambda$ on online voltage-to-voltage prediction performance with $10^5$ shots. The panels report mean relative $L_2$ error (left), empirical coverage (middle), and average and maximum interval widths (right).}
		
		\label{fig:online_analysis}
	\end{figure}
	
	\begin{figure}[t]
		
		\begin{center}
			\includegraphics[width=\columnwidth]{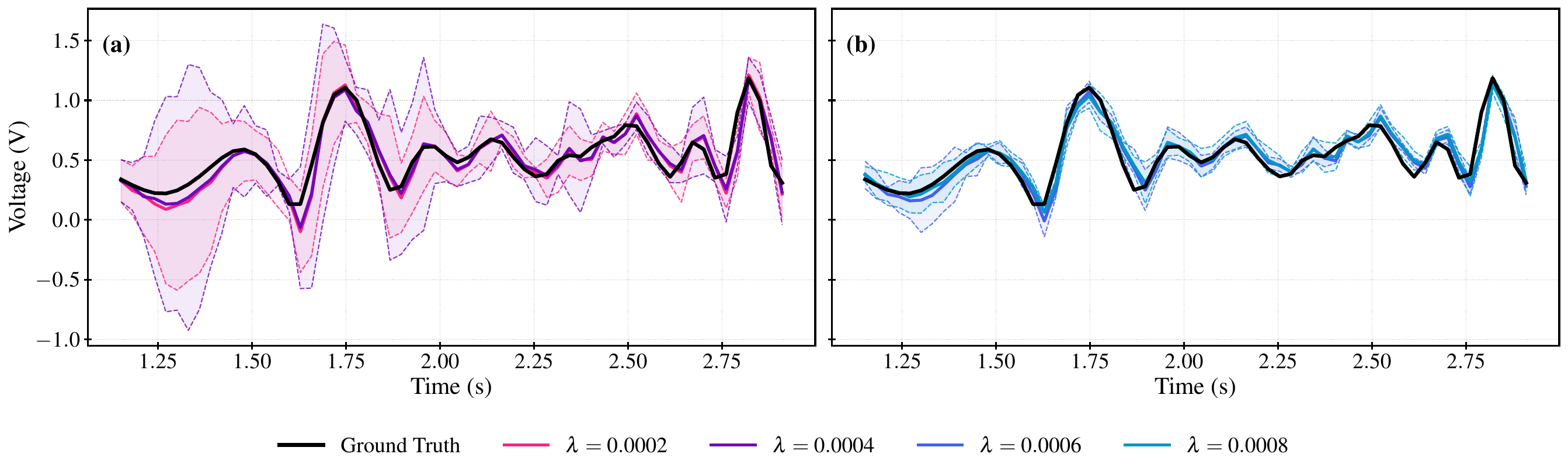}
		\end{center}
		
		\caption{Prediction intervals for one high-variance online voltage test trajectory under four depolarizing noise strengths. Panel (a) shows $\lambda\in\{0.0002,0.0004\}$, and panel (b) shows $\lambda\in\{0.0006,0.0008\}$ on the same vertical scale. The ground truth is shown in black; solid coloured lines denote ensemble means, and the corresponding dashed bounds and shaded regions denote the calibrated prediction intervals.}
		
		\label{fig:online_pred}
	\end{figure}

	Thus, for a new test trajectory block $\mathcal{B}_*=\{(\mathbf{u}_{*,j},\mathbf{y}_{*,j},s_{*,j})\}_{j=1}^{M}$ that is jointly exchangeable with the $K$ calibration blocks, Theorem~\ref{thm:coverage} gives
	\[
	\mathbb{E}\left[\frac{1}{M}\sum_{j=1}^{M}\mathbf{1}\left\{s_{*,j}\in C_\alpha(\mathbf{u}_{*,j},\mathbf{y}_{*,j})\right\}\right]\geq\frac{r}{(K+1)M},\qquad r=\left\lfloor(KM-1)(1-\alpha)\right\rfloor+1.
	\]
	The guarantee therefore applies to the expected fraction of covered next-step targets across the windows of a new trajectory. For $N_{\mathrm{test}}$ test trajectories, reported coverage is
	\[
	\widehat{\mathrm{Cov}}_{\mathrm{online}}=\frac{1}{N_{\mathrm{test}}}\sum_{i=1}^{N_{\mathrm{test}}}\frac{1}{M}\sum_{j=1}^{M}\mathbf{1}\left\{s_{ij}\in C_\alpha(\mathbf{u}_{ij},\mathbf{y}_{ij})\right\}.
	\]
	For the online task, each sliding window has a scalar target, so we compute the mean scalar relative $L_2$ error across all test windows:
	\[
	\frac{1}{N_{\mathrm{test}}M}\sum_{i=1}^{N_{\mathrm{test}}}\sum_{j=1}^{M}
	\frac{\left|\mu(\mathbf{u}_{ij})(\mathbf{y}_{ij})-s_{ij}\right|}
	{\left|s_{ij}\right|}.
	\]

    We evaluate the Causal Pointwise Transient Operator $(\mathcal{P})$ using windows consisting of five recent voltage measurements. The $L=4$ ensemble consists of 2-layer, 5-neuron QOrthoNNs, requiring 15 RBS rotation parameters in its largest quantum layer and a maximum circuit depth of 94 on IBM Heron. Under ideal simulation, the ensemble achieves a 5.49\% relative $L_2$ error and 90.11\% empirical query-averaged coverage relative to the nominal 90\% target.

 	The short input windows and shallow network produce compact circuits, making depolarizing-noise simulation tractable for individual inputs. However, each trajectory contributes many windows. Training on the complete window set and simulating noisy inference for every calibration and test window were outside our computational budget. We therefore use reduced datasets for the noise ablations.
 	
  	For the aggregate ablation in Figure~\ref{fig:online_analysis}, we draw reduced samples separately from the disjoint training, calibration, and test sets. For the qualitative ablation in Figure~\ref{fig:online_pred}, we restrict the data to high-variance signals and visualize one test trajectory to show how the intervals vary over time. Because the two ablations use different data selection and training procedures, we interpret them separately.

  	The mean relative $L_2$ error first increases from $\lambda=0$ to $\lambda=0.0002$, then decreases as $\lambda$ increases and falls below its noiseless value at $\lambda=0.0008$. Over the same range, empirical coverage dips slightly below the target at $\lambda=0.0002$ and then rises monotonically to approximately 99\%. This increase is not accompanied by systematic interval widening: the average width remains approximately constant across the nonzero noise levels, while the maximum width decreases at the two largest noise levels. For the selected high-variance trajectory, larger noise levels also produce narrower intervals. These trends suggest a possible noise-induced regularization effect, but the evidence is preliminary. We do not investigate its mechanism or whether hardware noise can systematically improve calibrated uncertainty estimates.

    \subsection{Hybrid Architecture}

	\begin{figure}[t]
		
		\begin{center}
			\includegraphics[width=\columnwidth]{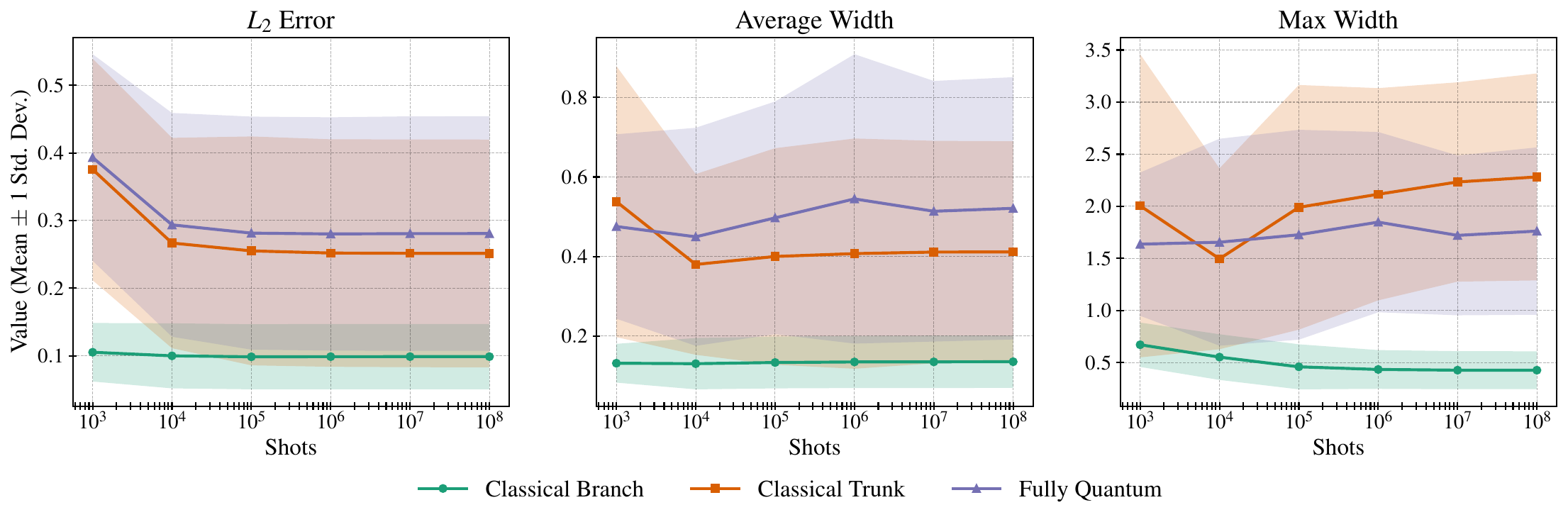}
		\end{center}
		
		\caption{Performance comparison of hybrid classical--quantum architectures on the antiderivative operator task. We evaluate three ensemble configurations: Fully Quantum (purple), Hybrid with Classical Trunk (orange), and Hybrid with Classical Branch (green). At each shot count, the lines report the mean across the four depolarizing noise strengths $\lambda\in\{0.0002,0.0004,0.0006,0.0008\}$, and the shaded regions denote $\pm1$ standard deviation across these noise strengths.}
		
		\label{fig:hybrid_arch}
	\end{figure}

    We compare fully quantum, classical-branch, and classical-trunk configurations on the antiderivative task as a proof-of-concept for hybrid ensemble execution. Although the branch and trunk have the same width and depth, the branch processes the higher-dimensional discretized input function and outputs a representation that is used across all trunk queries. As illustrated in Figure~\ref{fig:hybrid_arch}, the classical-branch hybrid achieves the lowest relative $L_2$ error and the narrowest prediction bands across the evaluated shot counts and noise strengths by removing quantum noise from this shared representation. In contrast, the classical-trunk hybrid improves only marginally on error and average width and increases the maximum interval width at larger shot counts, because it retains the noisy quantum branch. For this task, these results suggest that noise in the shared branch representation is the dominant source of degradation.

    \subsection{SPQC-Inspired Multiplexed Ensemble Execution}
    
    We evaluate the SPQC-inspired multiplexed circuit as a compact proof-of-concept for reducing the qubit duplication of parallel ensembling while retaining predictive accuracy and uncertainty quantification. We test the multiplexed circuit on the antiderivative task using an $L=4$ ensemble of 2-layer, 5-neuron QOrthoNNs. For the first branch layer, which maps the five raw inputs and one appended slack coordinate to width 5, the multiplexed circuit uses 9 qubits---six data qubits, one tomography ancilla, and two address qubits---compared with 28 qubits for naive parallel execution, where each of the four ensemble members uses six data qubits and one tomography ancilla. We restrict this evaluation to compact circuits with $L=4$ because noisy simulation becomes computationally expensive for larger circuits and ensembles. As illustrated in Figure~\ref{fig:SPQC_comparison}, multiplexed execution closely tracks the separately executed ensemble in relative $L_2$ error and average and maximum interval widths across the evaluated noise levels. The standard and multiplexed ensembles maintain empirical coverage near the nominal 90\% target across all evaluated noise levels.
    
    Table~\ref{tab:spqc_vs_normal_gates} reports the corresponding gate decomposition after transpilation to the IBM Heron topology. For $L=4$, the multiplexed circuit increases circuit depth and its CZ-gate count relative to one standard ensemble member while representing all four members in a shared circuit. Together with the qubit counts above, these results illustrate the qubit--depth trade-off in this compact setting; they do not constitute physical hardware validation or establish depth, sampling, or runtime scaling at larger ensemble sizes.

    \begin{figure}[t]

        \begin{center}
            \includegraphics[width=\columnwidth]{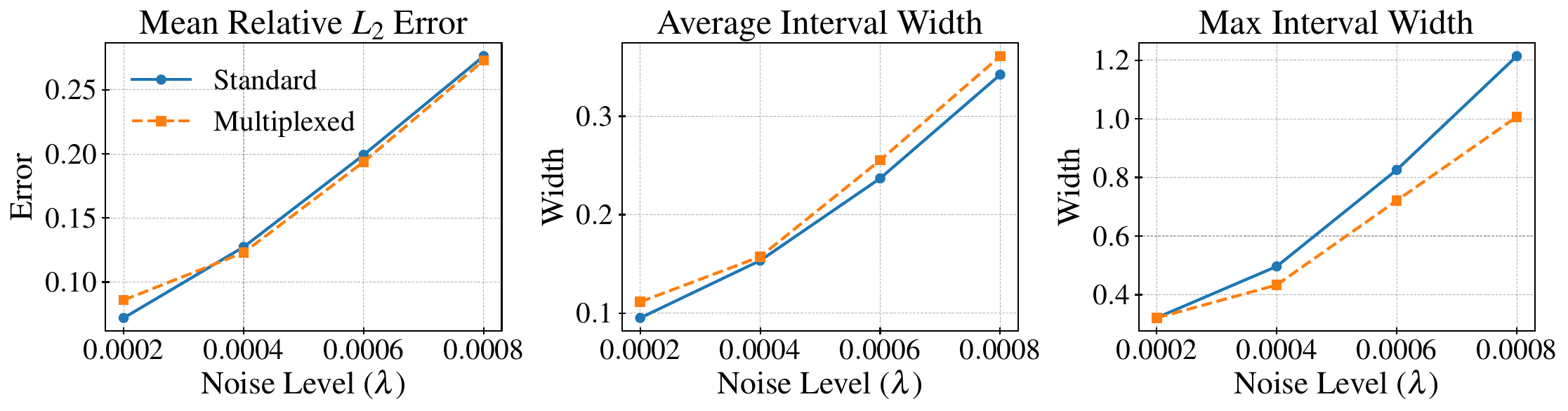}
        \end{center}

        \caption{Comparison of separately executed standard ensembling and SPQC-inspired multiplexed execution under depolarizing noise $\lambda$. Each standard ensemble member receives $10^5$ shots; the multiplexed circuit receives $10^5$ shots distributed across its four address branches. The multiplexed circuit closely tracks the standard ensemble across the reported metrics.}

        \label{fig:SPQC_comparison}
    \end{figure}

    \begin{table}[t]

        \caption{Transpiled circuit resources for a single quantum layer of one standard ensemble member and the corresponding SPQC-inspired multiplexed circuit with $L=4$. The Layers and Width columns describe the underlying QOrthoNN architecture; circuit depth and gate counts refer to the transpiled layer circuits.}
        \label{tab:spqc_vs_normal_gates}

        \begin{center}
            \begin{tabular}{lcccccccccc}
                \hline
                \textbf{Mode} & $d_u$ & $d_y$ & \textbf{Layers} & \textbf{Width} & \textbf{Depth} & \textbf{CZ} & \textbf{RX} & \textbf{RZ} & \textbf{SX} & \textbf{X} \\
                \hline

                Standard    & 5 & 1 & 2 & 5 & 95  & 62  & 76  & 62  & 21 & 1 \\
                Multiplexed & 5 & 1 & 2 & 5 & 272 & 173 & 144 & 125 & 44 & 1 \\

                \hline
            \end{tabular}
        \end{center}

    \end{table}

    \section{Limitations and Outlook}
    \label{sec:limitations}

    Our complexity analysis applies to square QOrthoNN hidden transformations under an idealized, depth-based execution model. It is therefore a layer-level asymptotic comparison rather than an end-to-end runtime result or a measured hardware speedup. The empirical evaluation is likewise based on ideal and noisy circuit simulations rather than physical quantum hardware. Because classical noisy-circuit simulation becomes expensive as circuit and ensemble size increase, the experiments are restricted to compact networks and small ensembles.
    
    Conditional on the trained ensemble and a fixed inference mechanism, the conformal guarantee requires jointly exchangeable calibration and test units with equal-length score vectors. The online experiment treats complete parent trajectories as exchangeable units and therefore does not provide a streaming guarantee for successive windows from a single continuously evolving trajectory. Moreover, the finite-sample lower bound may lie below the nominal $1-\alpha$ target when the number of calibration units is small, although it approaches the target as the calibration set grows.
    
    Finally, the empirical trends observed under depolarizing noise are preliminary and specific to the evaluated tasks, circuit sizes, and noise models. They do not establish a general noise-induced regularization mechanism or imply an improvement in the theoretical coverage bound.

	Within these limitations, our results show that resource-aware analysis of QOrthoNN hidden layers, ensemble-based uncertainty estimation, and finite-sample conformal calibration can be combined within a single operator-learning framework. The analysis identifies the computational regime in which favourable QOrthoNN circuit-depth scaling yields an asymptotic hidden-layer running-time advantage. The conformal result provides a finite-sample query-averaged marginal coverage guarantee, while the experiments demonstrate accurate predictions and empirical coverage near the target under ideal evaluation and selected noisy circuit simulations. The hybrid and SPQC-inspired experiments are secondary proof-of-concept studies of ensemble execution and do not resolve the trade-off between the runtime cost of sequential execution and the qubit cost of fully parallel execution. Evaluation on larger circuits and physical hardware, together with calibration methods for streaming data, distribution shift, and temporally varying hardware conditions, represents a natural next step toward scalable, uncertainty-aware quantum operator learning.

    \section*{Code and Data Availability}

    The code supporting the findings of this study is publicly available in the GitHub repository: \href{https://github.com/purav-0000/conformalized-quantum-deeponet}{https://github.com/purav-0000/conformalized-quantum-deeponet}. The repository also contains the data and scripts required to reproduce the experiments and results presented in this work.

    \section*{Acknowledgements}
    We acknowledge support from the National Science Foundation (DMS-2533878, DMS-2053746, DMS-2134209, ECCS-2328241, CBET-2347401, and OAC-2311848), the U.S.~Department of Energy (DOE) Office of Science Advanced Scientific Computing Research program (DE-SC0023161), the SciDAC LEADS Institute, and DOE Fusion Energy Sciences (DE-SC0024583).

    \bibliographystyle{tmlr}
    \bibliography{main}

    \appendix
    \section*{Appendix}

    \section{Detailed QOrthoNN Quantum State Evolution}
    \label{sec:quantum_state_evolution}

    For an $m\times n$ transformation, let $q=\max(m,n)$. Define $\mathbf{E}_n\in\mathbb{R}^{q\times n}$ to embed an input into the bottom $n$ coordinates,
    \[
    \mathbf{E}_n\mathbf{x}=(\underbrace{0,\ldots,0}_{q-n},x_1,\ldots,x_n)^{\top},
    \]
    and define $\boldsymbol{\Pi}_m\in\mathbb{R}^{m\times q}$ to retain the bottom $m$ coordinates,
    \[
    \boldsymbol{\Pi}_m\mathbf{v}=(v_{q-m+1},\ldots,v_q)^{\top}.
    \]
    The effective transformation implemented by the layer is $\mathbf{W}=\boldsymbol{\Pi}_m\mathbf{Q}(\boldsymbol{\theta})\mathbf{E}_n\in\mathbb{R}^{m\times n}$.
    When $m\geq n$, $\mathbf{W}^{\top}\mathbf{W}=\mathbf{I}_n$, and when $m\leq n$, $\mathbf{W}\mathbf{W}^{\top}=\mathbf{I}_m$. Following the QOrthoNN literature, we refer to both cases as orthogonal transformations.
    
    Here, $\mathbf{Q}(\boldsymbol{\theta})\in\mathbb{R}^{q\times q}$ is the square orthogonal transformation induced by $U(\boldsymbol{\theta})$ on the unary subspace, whereas $\mathbf{W}$ is the effective $m\times n$ layer transformation. In the square case, $\mathbf{E}_n=\boldsymbol{\Pi}_n=\mathbf{I}_n$, so $\mathbf{W}=\mathbf{Q}(\boldsymbol{\theta})$, consistent with the notation in Section~\ref{sec:method}.
    
    \begin{enumerate}
    	
    	\item \textbf{State preparation} --- Let $\mathbf{x}\in\mathbb{R}^n$ satisfy $\lVert\mathbf{x}\rVert_2=1$. The $q$ data qubits are initialized in the ground state, and an ancillary qubit is placed in uniform superposition:
    	\[
    	|\psi\rangle=\frac{1}{\sqrt{2}}|0\rangle|0\rangle^{\otimes q}+\frac{1}{\sqrt{2}}|1\rangle|0\rangle^{\otimes q}.
    	\]
    	A CNOT gate controlled by the ancilla acts on the first active input qubit, with index $q-n+1$, yielding
    	\[
    	|\psi\rangle=\frac{1}{\sqrt{2}}|0\rangle|0\rangle^{\otimes q}+\frac{1}{\sqrt{2}}|1\rangle|\mathbf{e}_{q-n+1}\rangle,
    	\]
    	where $|\mathbf{e}_k\rangle$ denotes the unary basis state whose $k$th qubit is in state $|1\rangle$.
    	
    	\item \textbf{Amplitude encoding} --- The data loader $S(\mathbf{x})$ encodes the input on the bottom $n$ qubits:
    	\[
    	S(\mathbf{x})|\mathbf{e}_{q-n+1}\rangle=\sum_{i=1}^{n}x_i|\mathbf{e}_{q-n+i}\rangle.
    	\]
    	Because the RBS gates composing $S(\mathbf{x})$ preserve the ground state, applying $I\otimes S(\mathbf{x})$ gives
    	\[
    	|\psi\rangle=\frac{1}{\sqrt{2}}|0\rangle|0\rangle^{\otimes q}+\frac{1}{\sqrt{2}}|1\rangle\sum_{i=1}^{n}x_i|\mathbf{e}_{q-n+i}\rangle.
    	\]
    	
    	\item \textbf{Pyramidal orthogonal transformation} --- The parameterized quantum unitary $U(\boldsymbol{\theta})$ acts as $\mathbf{Q}(\boldsymbol{\theta})$ on the unary subspace. Let
    	\[
    \widetilde{\mathbf{z}}=\mathbf{Q}(\boldsymbol{\theta})\mathbf{E}_n\mathbf{x}\in\mathbb{R}^{q}.
    	\]
    	The complete evolved data state is
    	\[
    U(\boldsymbol{\theta})S(\mathbf{x})|\mathbf{e}_{q-n+1}\rangle=\sum_{k=1}^{q}\widetilde{z}_k|\mathbf{e}_k\rangle.
    	\]
    	The classical output is obtained by retaining its bottom $m$ components:
    	\[
    \mathbf{z}=\boldsymbol{\Pi}_m\widetilde{\mathbf{z}}=\mathbf{W}\mathbf{x},
    \qquad
    z_j=\widetilde{z}_{q-m+j}=\sum_{i=1}^{n}W_{ji}x_i.
    	\]
    	When $m<n$, the remaining $n-m$ amplitudes stay in the quantum state but are not returned as layer outputs. Applying $I\otimes U(\boldsymbol{\theta})$ to the joint state therefore gives
    	\[
    |\psi\rangle=\frac{1}{\sqrt{2}}|0\rangle|0\rangle^{\otimes q}+\frac{1}{\sqrt{2}}|1\rangle\sum_{k=1}^{q}\widetilde{z}_k|\mathbf{e}_k\rangle.
    	\]
    	
    	\item \textbf{Tomography and sign extraction} --- Let $\mathbf{r}=(1/\sqrt{q},\ldots,1/\sqrt{q})\in\mathbb{R}^{q}$, and let $S(\mathbf{r})$ denote its data loader. Applying $I\otimes S^{\dagger}(\mathbf{r})$ gives
    	\[
    |\psi\rangle=\frac{1}{\sqrt{2}}|0\rangle|0\rangle^{\otimes q}+\frac{1}{\sqrt{2}}|1\rangle S^{\dagger}(\mathbf{r})\sum_{k=1}^{q}\widetilde{z}_k|\mathbf{e}_k\rangle.
    	\]
    	After applying an X gate to the ancilla and a CNOT gate controlled by the ancilla on the first data qubit, the state becomes
    	\[
    |\psi\rangle=\frac{1}{\sqrt{2}}|1\rangle|\mathbf{e}_1\rangle+\frac{1}{\sqrt{2}}|0\rangle S^{\dagger}(\mathbf{r})\sum_{k=1}^{q}\widetilde{z}_k|\mathbf{e}_k\rangle.
    	\]
    	Applying $I\otimes S(\mathbf{r})$ then yields
    	\[
    |\psi\rangle=\frac{1}{\sqrt{2}}|1\rangle\sum_{k=1}^{q}\frac{1}{\sqrt{q}}|\mathbf{e}_k\rangle+\frac{1}{\sqrt{2}}|0\rangle\sum_{k=1}^{q}\widetilde{z}_k|\mathbf{e}_k\rangle.
    	\]
    	Finally, a Hadamard gate on the ancilla produces
    	\[
    |\psi\rangle=\frac{1}{2}\sum_{k=1}^{q}\left(\widetilde{z}_k+\frac{1}{\sqrt{q}}\right)|0,\mathbf{e}_k\rangle+\frac{1}{2}\sum_{k=1}^{q}\left(\widetilde{z}_k-\frac{1}{\sqrt{q}}\right)|1,\mathbf{e}_k\rangle.
    	\]
    For the $j$th retained output, let $k_j=q-m+j$. Since $z_j=\widetilde{z}_{k_j}$,
    	\[
    	\sqrt{q}\left(\Pr[0,\mathbf{e}_{k_j}]-\Pr[1,\mathbf{e}_{k_j}]\right)
    	=
    \frac{\sqrt{q}}{4}\left(z_j+\frac{1}{\sqrt{q}}\right)^2
    	-
    \frac{\sqrt{q}}{4}\left(z_j-\frac{1}{\sqrt{q}}\right)^2
    	=
    z_j.
    	\]
    	
    \end{enumerate}

	\section{Nomenclature}
	\label{sec:nomenclature}
	
	\begingroup
	\normalsize
	\setlength{\LTleft}{0pt}
	\setlength{\LTright}{0pt}
	\renewcommand{\arraystretch}{1.05}
	\begin{longtable}{@{}p{0.21\textwidth}p{0.75\textwidth}@{}}
		\multicolumn{2}{@{}l}{\textbf{Data and Operator Formulation}}\\[0.2em]
		$\mathcal{G},\mathcal{G}_{\theta}$ & Target operator and its parameterized DeepONet approximation.\\
		$\mathcal{A}$ & Synthetic antiderivative operator.\\
		$\mathcal{F},\hat{\mathcal{F}}$ & Transient-to-Transient Operator and its learned surrogate.\\
		$\mathcal{H}$ & Causal Forecasting Operator.\\
		$\mathcal{P},\mathcal{P}_t$ & Causal Pointwise Transient Operator and its member indexed by time $t$.\\
		$\mathcal{U},\mathcal{Y}$ & Input- and output-trajectory function spaces in the power-system experiments.\\
		$\mathcal{U}_{\mathrm{past}},\mathcal{Y}_{\mathrm{future}}$ & Past-input and future-output function spaces for causal forecasting.\\
		$\mathcal{D}$ & Supervised training dataset.\\
		$\mathbf{u}_{i},\mathbf{u}$ & Discretized input function evaluated at fixed sensor locations.\\
		$\mathbf{u}_{ij},\mathcal{B}_{i},\mathcal{B}_{*}$ & Sliding-window input, complete parent-trajectory block, and generic test block in the online experiment.\\
		$\mathbf{y}_{ij},\mathbf{y}$ & Spatial or temporal query coordinate.\\
		$s_{ij}$ & Ground-truth operator value at query coordinate $\mathbf{y}_{ij}$.\\
		$\mathbf{s}_i$ & Discretized output trajectory in the power-system formulation.\\
		$d_u,d_y$ & Dimensions of the discretized input function and query coordinate.\\
		$N,N_{\mathrm{test}}$ & Numbers of training and test units, respectively.\\
		$M_i,M$ & Number of queries for unit $i$ and common number of queries per unit in the finite-sample analysis.\\[0.4em]
		
		\multicolumn{2}{@{}l}{\textbf{Network Architecture and Complexity}}\\[0.2em]
		$\mathbf{b},\mathbf{t}$ & Branch and trunk output vectors.\\
		$p$ & Common dimension of the branch and trunk output vectors.\\
		$\beta$ & Trainable scalar bias added to the branch--trunk inner product.\\
		$\mathcal{L}(\theta)$ & Training loss for the complete model.\\
		$\theta,\theta_\ell,d_{\theta}$ & Complete-model parameters, complete parameters of ensemble member $\ell$, and total number of trainable parameters.\\
		$d_{\mathrm{in}}$ & Raw input dimension of the first quantum hidden layer before appending the slack coordinate.\\
		$m,n$ & Output and input dimensions of a generalized QOrthoNN transformation; the square hidden-layer case has $m=n$.\\
		$\mathbf{W}$ & Effective $m\times n$ QOrthoNN layer transformation; $\mathbf{W}=\mathbf{Q}(\boldsymbol{\theta})$ in the square case.\\
		$\mathbf{E}_n,\boldsymbol{\Pi}_m$ & Input-embedding and output-selection matrices used to define the effective transformation $\mathbf{W}$.\\
		$\bar{\mathbf{x}},\mathbf{x}$ & Min--max-normalized raw input and unit-norm input of a QOrthoNN layer.\\
		$\widetilde{\mathbf{z}},\mathbf{z}$ & Transformed $q$-dimensional vector before output selection and retained $m$-dimensional layer output.\\
		$P$ & Number of available classical processors in the work--depth comparison.\\
		$W,D$ & Work and computational depth in Brent's scheduling bound.\\
		
		$\delta$ & Per-layer $\ell_\infty$ tomography tolerance.\\[0.4em]
		
		\multicolumn{2}{@{}l}{\textbf{Quantum Circuit and Ensemble}}\\[0.2em]
		$L$ & Number of independently trained Quantum DeepONet ensemble members.\\
		$q$ & Number of data qubits in the generalized circuit, $q=\max(m,n)$, excluding the ancilla.\\
		$\boldsymbol{\theta}$ & Vector of RBS rotation angles in a QOrthoNN layer.\\
		$\boldsymbol{\theta}_{\mathrm{B}},\boldsymbol{\theta}_{\mathrm{T}}$ & RBS-angle vectors for the branch and trunk QOrthoNN layers, respectively.\\
		$\mathbf{x}^{(\ell)},\boldsymbol{\theta}^{(\ell)}$ & Layer input and RBS-angle vector for multiplexed ensemble member $\ell$.\\
		$\phi$ & Rotation angle of a single RBS gate.\\
		$S(\mathbf{x})$ & Data-dependent unitary that amplitude-encodes $\mathbf{x}$ in the unary subspace.\\
		$U(\boldsymbol{\theta})$ & Pyramidal parameterized unitary implementing the QOrthoNN transformation.\\
		$|\psi\rangle$ & Quantum state during QOrthoNN evaluation; the included registers are specified in context.\\
		$|\mathbf{e}_i\rangle$ & Unary computational basis state whose $i$-th qubit in the data register is in state $|1\rangle$.\\
		$\mathbf{r}$ & Uniform reference vector used for tomography and sign extraction.\\
		$N_{\mathrm{shots}}$ & Number of quantum measurement shots.\\
		$\mathbf{Q}(\boldsymbol{\theta})$ & Square orthogonal transformation induced by $U(\boldsymbol{\theta})$ on the unary subspace.\\
		$\lambda$ & Noise strength of the simulated depolarizing channel.\\[0.4em]
		$\rho,\mathcal{E}_A,I_A,\operatorname{Tr}_A$ & Quantum density matrix, local depolarizing channel, identity, and partial trace on subsystem $A$.\\[0.4em]
		
		\multicolumn{2}{@{}l}{\textbf{Uncertainty Quantification}}\\[0.2em]
		$\mu(\mathbf{u})(\mathbf{y}),\mu_{ij}$ & Ensemble mean prediction and its value for query $j$ of unit $i$.\\
		$\sigma(\mathbf{u})(\mathbf{y})$ & Ensemble standard deviation measuring model disagreement.\\
		$\tilde{\sigma},\tilde\sigma_{ij},\epsilon$ & Regularized ensemble standard deviation, its indexed value, and the positive stabilizer in $\tilde{\sigma}=\sigma+\epsilon$.\\
		$\mathcal{D}_{\mathrm{cal}},\mathcal{D}_{\mathrm{cal},i}$ & Calibration dataset and its $i$-th complete calibration unit.\\
		$K,N_{\mathrm{cal}}$ & Number of calibration units and total number of calibration scores.\\
		$R,R_{ij}$ & Generic nonconformity score and the score for query $j$ of unit $i$.\\
		$\mathbf{R}_i$ & Complete vector of query-level scores for unit $i$.\\
		$R^{\mathrm{cal}}_{(k)},R_{(k)},R^{-i}_{(k)}$ & $k$-th order statistics of the calibration pool, augmented pool, and augmented pool after removing unit $i$, respectively.\\
		$a,r,\gamma$ & Quantile interpolation variables, with $a=(N_{\mathrm{cal}}-1)(1-\alpha)$, $r=\lfloor a\rfloor+1$, and $\gamma=a-\lfloor a\rfloor$.\\
		$\eta$ & Rounding remainder in $r=(KM-1)(1-\alpha)+\eta$.\\
		$\widehat q,\widehat q_{-i}$ & Calibration threshold and its leave-one-unit-out counterpart used in the proof.\\
		$C_\alpha(\mathbf{u},\mathbf{y})$ & Conformal prediction interval at input $\mathbf{u}$ and query coordinate $\mathbf{y}$.\\
		$w(\mathbf{u})(\mathbf{y})$ & Conformal interval half-width shown in Figure~\ref{fig:main_framework}.\\
		$\alpha$ & Target miscoverage level.\\
		$J$ & Query index drawn uniformly from $\{1,\ldots,M\}$.\\
		$V_i$ & Fraction of unit $i$ covered by its leave-one-unit-out threshold in the proof.\\
		$\widehat{\mathrm{Cov}},\widehat{\mathrm{Cov}}_{\mathrm{online}}$ & Empirical query-averaged coverage in the general and online settings.\\[0.4em]
		
		\multicolumn{2}{@{}l}{\textbf{Experiments}}\\[0.2em]
		$k_l,d(x,x'),l$ & Covariance kernel, distance between input coordinates, and kernel length scale in the synthetic benchmarks.\\
		$d$ & State-space dimensionality in the power-system operators.\\
		$T,\tau,T_f,t_0,\Delta t$ & Trajectory horizon, history-window length, forecast horizon, contingency time, and next-step interval, respectively.\\
		$f_1,\ldots,f_5$ & Dominant output frequencies used in the Fourier-augmented trunk coordinates.\\
		$f_c,f_s$ & Cutoff and sampling frequencies used to filter active-power trajectories.\\
		$\gamma_{\mathrm{LR}}$ & Multiplicative learning-rate decay factor.\\
		$T_1,T_2$ & Qubit relaxation and dephasing times reported in the backend snapshots.\\
	\end{longtable}
	\endgroup

    \newpage
    \section{Training Hyperparameters and Dataset Splits}\label{sec:train_hp}

    We list the hyperparameters used during training in Table~\ref{tab:training_hyperparams}. For experiments with learning-rate decay, we use a multiplicative schedule with factor $\gamma_{\mathrm{LR}}$ at each optimization step and hold the learning rate at the reported minimum.

    \begin{table}[t]

        \caption{Training hyperparameters. ``Full'' indicates full-batch gradient descent. ``--'' indicates a constant learning rate (no scheduler). Synthetic rows report training, calibration, and test counts, whereas power-system rows report split fractions. $\gamma_{\mathrm{LR}}$ is the learning-rate decay factor.}
		\label{tab:training_hyperparams}
		
        \begin{center}
            \small
            \begin{tabular}{lccccccccc}
                \hline
                \textbf{Experiment} & \textbf{Train} & \textbf{Cal.} & \textbf{Test} & \textbf{Batch} & \textbf{Iters} & \textbf{LR} & \textbf{Min LR} & $\gamma_{\mathrm{LR}}$ & \textbf{Loss} \\
                \hline

                Antiderivative & 200  & 50  & 50  & Full & 30k & $1\mathrm{e}{-3}$ & --              & --   & MSE \\
                Advection      & 1000 & 200 & 200 & Full & 40k & $1\mathrm{e}{-3}$ & $5\mathrm{e}{-4}$ & 0.9999653 & MSE \\
                Offline V-to-P & 0.8  & 0.1 & 0.1 & 64   & 40k & $5\mathrm{e}{-3}$ & $5\mathrm{e}{-4}$ & 0.9998465 & Rel. $L_2$ \\
                Offline V-to-V & 0.8  & 0.1 & 0.1 & 256  & 40k & $5\mathrm{e}{-4}$ & --              & --   & MSE \\
                Online V-to-V  & 0.8  & 0.1 & 0.1 & 256  & 15k & $1\mathrm{e}{-2}$ & --              & --   & MSE \\

                \hline
            \end{tabular}
        \end{center}

    \end{table}

    \section{Multinomial Sampling Versus Per-Shot Simulations}
    \label{sec:multinomial}

    To accelerate the simulation of quantum circuits, we employ multinomial sampling from the final statevector or density matrix probability distribution, rather than simulating each shot individually. This single simulation followed by sampling is computationally efficient.

    To validate this approach, we compare its results against a full per-shot simulation, which applies noise at every step for every shot. Figure~\ref{fig:appendix_c_compare} shows this comparison for the antiderivative operator. The results show close agreement across relative $L_2$ error, coverage, average width, and maximum width, supporting multinomial sampling as an accurate and efficient substitute for the more costly per-shot simulation.

    \begin{figure}[t]
        \centering
        \includegraphics[width=1.0\textwidth]{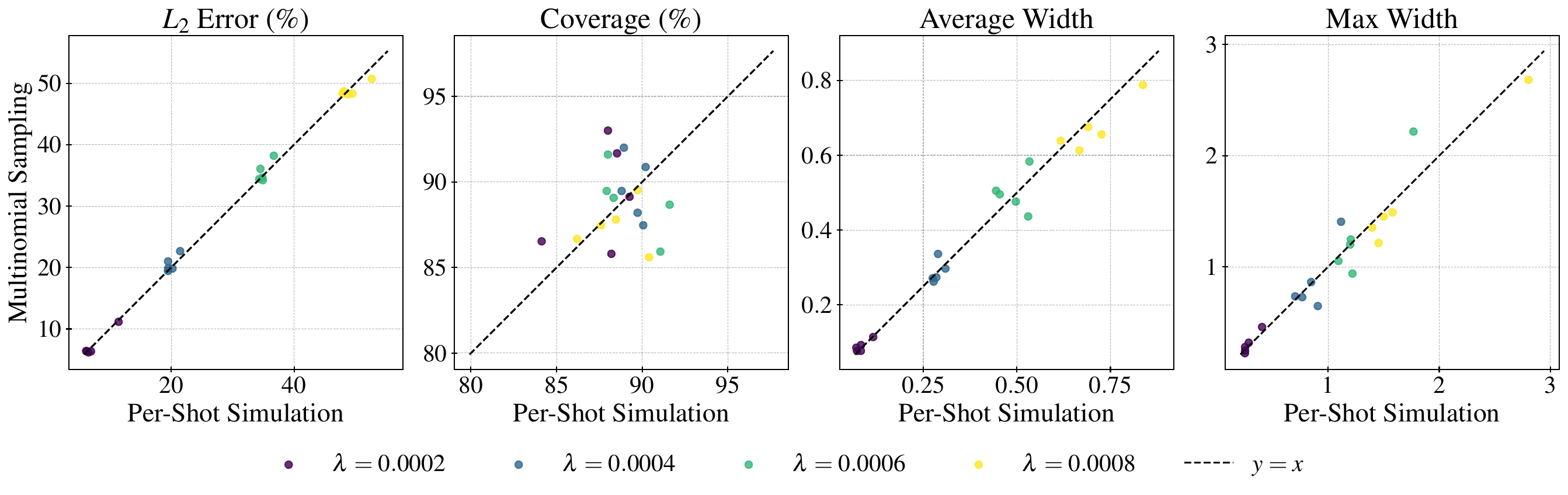}
        \vspace{-2em}
        \caption{Validation of multinomial sampling. Parity plots comparing the reported metrics from per-shot simulation (horizontal axis) and multinomial sampling (vertical axis). The dashed line represents equality.}
        \label{fig:appendix_c_compare}
    \end{figure}

    \clearpage
    \section{System Snapshots}\label{sec:snapshots}

    The device-calibrated composite noise models were constructed from calibration data for three Qiskit backends: \texttt{fake\_brisbane}, \texttt{fake\_torino}, and \texttt{fake\_marrakesh}. All simulations were transpiled to the basis gate set and coupling map of the respective backend. Table~\ref{tab:backend_specs} reports the qubit counts, basis gate sets, coupling sizes, coherence times, and error rates taken from these snapshots.

    \begin{table}[htbp]

        \caption{Backend characteristics for the simulated quantum devices. Coherence times are reported in microseconds, and readout and gate errors are reported as percentages. The $T_1$, $T_2$, and readout-error statistics are computed across qubits, whereas gate-error statistics are computed across basis-gate entries. Only the mean is reported for two-qubit gate errors. Two-qubit means exclude entries with gate error $1.0$, which IBM uses for undefined or stale calibration estimates.}
        \label{tab:backend_specs}

        \begin{center}
            \small

            \begin{tabular}{lccc}
                \hline
                \textbf{Backend} & \textbf{Qubits} & \textbf{Basis Gates} & \textbf{Coupling} \\
                \hline

                fake\_brisbane  & 127 & \{ecr, id, rz, sx, x\} & 144 edges \\
                fake\_torino    & 133 & \{cz, id, rz, sx, x\}  & 300 edges \\
                fake\_marrakesh & 156 & \{cz, id, rz, sx, x\}  & 352 edges \\

                \hline
            \end{tabular}

            \vspace{0.6em}

            \begin{tabular}{lcccccc}
                \hline
                \textbf{Backend} & $T_1$ avg ($\mu\mathrm{s}$) & $T_1$ min ($\mu\mathrm{s}$) & $T_1$ max ($\mu\mathrm{s}$) & $T_2$ avg ($\mu\mathrm{s}$) & $T_2$ min ($\mu\mathrm{s}$) & $T_2$ max ($\mu\mathrm{s}$) \\
                \hline
                fake\_brisbane  & 233.71 & 9.94 & 425.17 & 160.58 & 17.14 & 411.53 \\
                fake\_torino    & 173.61 & 3.16 & 307.73 & 145.19 & 6.13 & 349.88 \\
                fake\_marrakesh & 207.13 & 6.80 & 499.05 & 147.00 & 0.18 & 534.79 \\

                \hline
            \end{tabular}

            \vspace{0.6em}

            \begin{tabular}{lccc}
                \hline
                \textbf{Backend} & \textbf{Readout Err (avg)} & \textbf{Readout Err (min)} & \textbf{Readout Err (max)} \\
                \hline

                fake\_brisbane  & 3.11 & 0.56 & 23.95 \\
                fake\_torino    & 4.67 & 0.51 & 56.52 \\
                fake\_marrakesh & 2.70 & 0.27 & 47.39 \\

                \hline
            \end{tabular}

            \vspace{0.6em}

            \begin{tabular}{lcccc}
                \hline
                \textbf{Backend} & \textbf{1Q Err (avg)} & \textbf{1Q Err (min)} & \textbf{1Q Err (max)} & \textbf{2Q Err (avg)} \\
                \hline

                fake\_brisbane  & 0.11 & 0.00 & 14.09 & 1.28 \\
                fake\_torino    & 0.06 & 0.00 & 1.74  & 0.72 \\
                fake\_marrakesh & 0.06 & 0.00 & 4.20  & 0.72 \\

                \hline
            \end{tabular}

        \end{center}

    \end{table}

\end{document}